\documentclass[11pt]{article}
\usepackage{geometry}
\usepackage{Shorthands}

\title{Adversarially Robust Stability Certificates can be Sample-Efficient}

\usepackage{authblk}
\author[1]{Thomas T.C.K.\ Zhang}
\author[2]{Stephen Tu}
\author[3]{Nicholas M.\ Boffi}
\author[4,2]{\authorcr{}Jean-Jacques E.\ Slotine}
\author[1,2]{Nikolai Matni}
\affil[1]{Department of Electrical and Systems Engineering, University of Pennsylvania}
\affil[2]{Google Brain Robotics}
\affil[3]{Courant Institute of Mathematical Sciences, New York University}
\affil[4]{Nonlinear Systems Laboratory, Massachusetts Institute of Technology}

\usepackage{times}

\begin{document}

\setlength{\belowcaptionskip}{-5pt}
\captionsetup{belowskip=-5pt}

\setlength{\abovedisplayskip}{6pt}
\setlength{\belowdisplayskip}{6pt}
\setlength{\abovedisplayshortskip}{6pt}
\setlength{\belowdisplayshortskip}{6pt}

\maketitle

\begin{abstract}%
Motivated by bridging the simulation to reality gap in the context of safety-critical systems, we consider learning adversarially robust stability certificates for unknown nonlinear dynamical systems.  In line with approaches from robust control, we consider additive and Lipschitz bounded adversaries that perturb the system dynamics.  We show that under suitable assumptions of incremental stability on the underlying system, the statistical cost of learning an adversarial stability certificate is equivalent, up to constant factors, to that of learning a nominal stability certificate. Our results hinge on novel bounds for the Rademacher complexity of the resulting adversarial loss class, which may be of independent interest.  To the best of our knowledge, this is the first characterization of sample-complexity bounds when performing adversarial learning over data generated by a dynamical system.  We further provide a practical algorithm for approximating the adversarial training algorithm, and validate our findings on a damped pendulum example.%
\end{abstract}


\section{Introduction}

A challenge to the deployment of modern robotic systems to real-world settings is the overall lack of formal safety guarantees. While controller design for complex robotic systems has received much attention, comparatively less effort has been devoted to verifying the safety of the resulting closed-loop system. Without broadly applicable tools for certifying \textit{a-priori} guarantees, it is difficult to justify deploying these methods in applications where safety is paramount, regardless of the impressive performance that they achieve in simulation or controlled laboratory settings.

An important component of ensuring real-world safety is verifying the stability of a closed-loop system from trajectory data. 
While recent work \citep{boffi2020learning} proposes and analyzes a learning-based approach to this problem,
a fundamental limitation of the prior art
is that 
learning a stability certificate with failure probability of less than e.g.,~$1\%$ for high-dimensional systems requires on the order of tens of thousands of trajectories.  Realistically, such a large amount of trajectory data can only be collected using a simulation environment.
Therefore, in order for a learned certificate to be meaningful for real-world hardware, it is essential for it to be \emph{robust} to modeling errors between simulation and reality, i.e., robust to the so-called sim-to-real gap.  

While bridging the sim-to-real gap has traditionally been addressed via domain randomization \citep{tobin2017domain}, we take inspiration from the robust control literature, and tackle this challenge by developing an approach for \emph{adversarial learning} of stability certificates for dynamical systems.
We show that under suitable conditions on the underlying system, requiring that a learned certificate is robust to adversarial perturbations that \emph{enter the dynamics} carries little additional statistical overhead.
Taking inspiration from \citet{boffi2020learning}, we prove our results by converting the \emph{robust} stability certification problem into an adversarial learning problem, and subsequently bounding the Rademacher complexity of the resulting adversarial loss class.  
To the best of our knowledge, this is the first characterization of sample-complexity bounds when performing adversarial learning over data generated by a dynamical system.
Our results build upon and extend a line of work which shows that underlying system-theoretic properties translate into the difficulty (or ease) of learning over data generated by dynamical systems (see e.g., \citet{tsiamis2020sample, tsiamis2021linear, lee2021adversarial, tu2021sample} and references therein).
We further provide a practical algorithm for approximating the adversarial training algorithm, and show that adversarially trained certificates are robust to various types of model mis-specification on a damped pendulum example.
Our results are presented in continuous time; however, they readily admit discrete time analogues, which are detailed in Appendix~\ref{sec: DT results}.

\subsection{Related Work} 
%
%
Our work draws upon and unifies tools from three areas: (i) learning safety certificates from data, (ii) adversarial robustness, and (iii) statistical learning theory.

\paragraph{Learning safety certificates}
A wide body of work addresses learning Lyapunov \citep{giesl20lyapunov, kenanian19switchedlinear, chen20PWALyap, richards18lyapunov, manek19learningstable, chang19neuralcontrol, ravanbakhsh19lyapunov} and barrier \citep{taylor19cbf, robey20cbf, jin20neural} functions, as well as contraction metrics \citep{singh19learning, ccm_orig, sumeet_icra} and contracting vector fields \citep{sindhwani18vectorfields, elkhadir19teleop} from data.  While the generality and strength of guarantees provided vary (see the literature review of \citet{boffi2020learning} for a detailed exposition), all of the aforementioned works consider nominally specified systems without uncertainty, whereas our approach explicitly considers perturbations that can capture model uncertainty and process noise.

\paragraph{Adversarial robustness}
Traditional approaches \citep{szegedy2013intriguing,madry2017towards,zhang2019theoretically,kurakin2016adversarial} to adversarial learning consider worst-case perturbations to the data during training, i.e., the data is perturbed \emph{after it has been generated}. While such a perturbation model is meaningful in the image classification setting for which adversarial robust training methods were originally developed, it does
not immediately translate to the dynamic setting that we consider, where the adversary may be used to capture model uncertainty or process noise.  In particular, our adversarial model perturbs the dynamical system which generates the data, a perspective that is more in line with traditional robust control methods.  We further show that under suitable stability assumptions on the underlying dynamical system, there is no additional statistical cost to adversarial training, in contrast to results showing that in the traditional setting, adversarial learning algorithms require more data than their nominal counterparts \citep{schmidt2018adversarially}. 

Most directly relevant to our work are adversarial deep reinforcement learning methods which learn policies that are robust to various classes of disturbances, such as adversarial observations \citep{torabi2019generative, gleave2019adversarial}, rewards \citep{fu2017learning, ho2016generative}, direct disturbances to the system \citep{pinto2017robust}, or combinations thereof \citep{lutter2021robust}. 
%
Nevertheless, there remains a paucity of theoretical guarantees on the generalization error, and thus sample-efficiency, of such learned policies.

\paragraph{Statistical learning theory} While such statistical guarantees, to the best of our knowledge, do not exist for adversarial reinforcement learning, the generalization error of an adversarially trained classifier has been studied using uniform convergence \citep{yin2019rademacher, attias2019improved, montasser2019vc}.  
%
While our results also rely on uniform convergence,
our analysis departs from this existing line of work by allowing adversaries to influence dynamical systems.

\section{Problem Framework}\label{sec: problem framework}

\subsection{Nominal Stability Certificates}\label{ssec: nominal stability certificates}

We begin by reviewing the problem setting and results from \citet{boffi2020learning}.  We assume that the underlying dynamical system is a continuous-time, autonomous system of the form $\dotx = f(x)$, where $f$ is continuous and unknown, and that the state $x \in \R^p$ is fully observed. Let $\calX \subset \R^p$ be a compact set and $\calT \subseteq \R^+$ be the maximum interval such that a unique solution $\varphi_t(\xi)$ exists for all times $t \in \calT$ and initial conditions $\xi \in \calX$, where $\varphi_t(\xi)$ is the map to the state at time $t$ given initial condition $\xi$. We assume that we have access to $n$ trajectories initialized from randomly sampled initial conditions. That is, we are given $\curly{\varphi_t(\xi_i)}_{i \in [n],\; t \in \calT}$, where $\xi_1,\dots,\xi_n \overset{\mathrm{i.i.d.}}{\sim} \calD$ and $\calD$ is a distribution over $\calX$. For simplicity, we assume that we can precisely differentiate $\varphi_t(\xi)$ with respect to time (in practice, we can estimate $\dot{\varphi}_t(\xi)$ numerically).

Let $\calV$ be a class of
continuously differentiable candidate Lyapunov functions $V: \R^p \to \R_{\geq 0}$ satisfying $V(0) = 0$. Fixing a constant $\eta > 0$, we define a scalar violation function $h: \calX \times \calV$ as:
\begin{align}\label{eq:h_nom}
    h(\xi, V) := \sup_{t \in \calT} \ip{\nabla V\paren{\varphi_t(\xi)} , f\paren{\varphi_t(\xi)} } + \eta V\paren{\varphi_t(\xi)}.
\end{align}
The violation function $h(\xi, V)$ scans the Lyapunov decrease condition for exponential stability with rate $\eta$ over the trajectory initialized at $\xi$, and returns the maximal value. Observe that if $h(\xi, V) \leq 0$, then $V$ certifies exponential stability along the trajectory $\varphi_t(\xi)$, $t \in \calT$. The nominal stability certification problem is therefore equivalent to the following feasibility problem:
\begin{align}\label{eq: true nominal feasibility}
    \text{Find}_{V \in \calV} \text{ s.t. } h(\xi, V) \leq 0 \quad \forall \xi \in \calX.
\end{align}
In general, various choices of $\calV$ and $h(\xi,V)$ can encode different notions of stability and accompanying certificates (see \cite{boffi2020learning} for more details). To search for a $V$ that satisfies the above optimization problem given finite data, we solve the following feasibility problem:
\begin{align}\label{eq: empirical nominal feasibility}
    \text{Find}_{V \in \calV} \text{ s.t. } h(\xi_i, V) \leq -\tau, \quad i=1,\dots,n,
\end{align}
where $\tau > 0$ is a margin that ensures generalization of the learned stability certificate $V$ on unseen trajectories. Let $\hat{V}_n$ denote a solution to \eqref{eq: empirical nominal feasibility} and define the nominal generalization error of $\hat{V}_n$ as
\begin{align}\label{eq: nominal gen error}
    \err(\hat{V}_n) := \prob_{\xi \sim \calD} \brac{ h\paren{\xi, \hat{V}_n} > 0}.
\end{align}
The nominal error \eqref{eq: nominal gen error} characterizes the probability that $\hat{V}_n$ fails to certify stability along a new trajectory with initial condition sampled from $\calD$. In \cite{boffi2020learning}, it is shown that for general classes of $\calV$, $\err(\hat{V}_n)$ decays at a rate $\tilde{O}(k/n)$, where $k$ captures the effective degrees of freedom of the stability function class $\calV$ and $\tilde{O}$ suppresses polylog dependence on $n$ and fixed problem parameters.

\subsection{Adversarially Robust Stability Certificates}\label{ssec: advrobust stability certificates}

We now consider the stability certification problem under the presence of adversarial perturbations. Consider the following two tubes of perturbed trajectories\footnote{
Existence, uniqueness, and completeness of the perturbed trajectories over the interval $[0,T]$ can be guaranteed under various
assumptions. As an example, the set \eqref{def: norm-bounded pert ball}
is well-defined if $f(x)$ is assumed to be continuous in $x$ and input-to-state stable such that $\tilde\varphi_t \in S$ for all $t\geq 0$ \citep[Prop. C.3.5]{sontag2013mathematical}.
Similarly, the set \eqref{def: lipschitz-bounded pert ball} is 
well-defined if we additionally assume that $f(x)$ is globally Lipschitz in $x$ \citep[Prop. C.3.8]{sontag2013mathematical}. We note that alternative assumptions on $f(x) + \delta(x)$ can be used to ensure completenes, e.g., that $f(x) + \delta(x)$ is stable in the sense of Lyapunov for all admissible $\delta$.
}:
\begin{align}
\Delta^u_\varepsilon(\xi) &:= \curly{\tvarphi: \dot{\tilde{\varphi}}_{t} = f(\tilde{\varphi}_t) + \delta_t,\; \tilde{\varphi}_0 = \xi, \; \norm{\delta_t}_2\leq \varepsilon,\, t \mapsto \delta_t \text{ is locally integrable}}, \label{def: norm-bounded pert ball} \\
\Delta^x_\varepsilon(\xi) &:= \curly{\tvarphi: \dot{\tvarphi}_{t} = f(\tilde{\varphi}_t) + \delta(\tilde{\varphi}_t),\; \tilde{\varphi}_0 = \xi, \; \norm{\delta(\tilde{\varphi}_t)}_2 \leq \varepsilon \norm{\tilde{\varphi}_t}_2}. \label{def: lipschitz-bounded pert ball}
\end{align}



Intuitively, $\Delta^u_\varepsilon(\xi)$ is the tube of perturbed trajectories initialized at $\xi$ for which an additive adversary has an instantaneous norm budget of $\varepsilon$ to perturb the dynamics. Analogously, $\Delta^x_\varepsilon(\xi)$ is the tube of perturbed trajectories initialized at $\xi$ for which the adversary satisfies $\varepsilon$-linear growth. We refer adversaries of the form \eqref{def: norm-bounded pert ball} as \emph{norm-bounded}, and adversaries of the form \eqref{def: lipschitz-bounded pert ball}, misnomer notwithstanding, as \emph{Lipschitz}. Indeed, given $\delta(0) = 0$, $\delta(x)$ being $\varepsilon$-Lipschitz is implied by $\varepsilon$-linear growth. The norm-bounded adversary can be used to capture small disturbances to the dynamics, such as process noise, while the Lipschitz adversary can be used to capture \textit{model} error between the training and test trajectories. 
We also define an adversary that is the linear combination of the norm-bounded and Lipschitz adversaries, which leads to the following tube of perturbed trajectories:
\begin{align}\label{def: combo of adv}
    \Delta^{x,u}_{\varepsilon_x,\varepsilon_u}(\xi) := \curly{\tvarphi: \tvarphi_{t} = f(\tvarphi_t) + \delta^x(\tvarphi_t) + \delta^u_t,\; \tvarphi_0 = \xi,\; \norm{\delta^x(\tvarphi_t)}_2 \leq \varepsilon_x \norm{\tvarphi_t}_2,\; \norm{\delta^u_t}_2 \leq \varepsilon_u }.
\end{align}
Here, the $\delta^u_t$ are additionally assumed to be locally integrable with respect to $t$.  
The tube \eqref{def: combo of adv} of perturbed trajectories defines a natural way of capturing the sim-to-real gap through the effects of both unmodeled dynamics ($\delta^x$) and process noise ($\delta^u$).

In order to accommodate additive disturbances in our stability analysis, we modify the violation function \eqref{eq:h_nom} to certify \emph{practical stability} \citep{lin1995various}, i.e., convergence to a ball about the origin.  To that end, for $\nu \geq 0$, define the adversarial violation function:
\begin{align}\label{eq: advrobust loss function}
    \tildh_{\nu}(\xi, V) := \sup_{\tvarphi \in \Delta_\varepsilon} \sup_{t \in \calT} \ip{\nabla V(\tvarphi_t(\xi)), \dot{\tvarphi}_t(\xi)} + \eta V(\tvarphi_t(\xi)) - \nu.
\end{align}
With this definition, finding an adversarially robust certificate of practical stability from data can be posed as solving the following feasibility problem analogous to \eqref{eq: empirical nominal feasibility}:
\begin{align}\label{eq: empirical advrobust feasibility}
    \text{Find}_{V \in \calV} \text{ s.t. } \tildh_{\nu}(\xi_i, V) \leq -\tau, \quad i=1,\dots,n.
\end{align}
Letting $\tilde{V}_n$ be the solution to \eqref{eq: empirical advrobust feasibility}, we consider the analogous generalization error to \eqref{eq: nominal gen error}:
\begin{align}\label{eq: advrobust gen error}
    \err(\tilde{V}_n) := \prob_{\xi \sim \calD} \brac{ \tildh_{\nu}\paren{\xi, \tilde{V}_n} > 0}.
\end{align}
Our goal is to show that the fast rates $\tilde{O}(k/n)$ enjoyed in the nominal setting are \textit{preserved} in the adversarial setting when the underlying system satisfies certain incremental stability conditions.

\section{Sample Complexity of Learning Adversarially Robust Stability Certificates} \label{sec: sample complexity}

We first introduce our main stability assumption on the system dynamics.

\begin{assumption}[Stability in the sense of Lyapunov] \label{assumption: bounded adversary traj}
Fix a perturbation set $\Delta(\cdot)$.
There exists a compact set $S \subseteq \R^p$ such that $\tvarphi_t(\xi) \in S$ for all $\xi \in \calX$, $t \in \calT$, and $\tvarphi_t(\cdot) \in \Delta(\xi)$.
\end{assumption}
For norm-bounded adversaries \eqref{def: norm-bounded pert ball}, this assumption is satisfied if the underlying nominal dynamics are input-to-state stable \citep{lin1995various}.  For Lipschitz \eqref{def: lipschitz-bounded pert ball} and combined \eqref{def: combo of adv} adversaries, additional care must be taken to ensure that $f(x) + \delta^x(x)$ remains input-to-state stable for all admissible $\delta^x(x)$.

We further make the following regularity assumptions on the certificate function class $\calV$.
\begin{assumption}[Regularity of $\calV$] \label{assumption: bounded constants}
There exists constants $L_V$, $L_{\nabla V}$ such that
for every $V \in \calV$, the maps $x \mapsto V(x)$ and $x \mapsto \ip{\nabla V(x), f(x)}$ over $x \in S$ are $L_V$ and $L_{\nabla V}$-Lipschitz, respectively.
\end{assumption}
%
Under Assumptions~\ref{assumption: bounded adversary traj} and \ref{assumption: bounded constants} and the continuity of the nominal dynamics $f(x)$, there exist constants $B_V$, $B_{\nabla V}$, and $B_{\tildh}$ such that
\begin{align*}
    \sup_{V \in \calV} \sup_{x \in S} \abs{V(x)} \leq B_V, \ \sup_{V \in \calV} \sup_{x \in S} \norm{\nabla V(x)}_2 \leq B_{\nabla V}, \ \sup_{V \in \calV} \sup_{\xi \in \calX}\abs{\tilde h(\xi, V)} \leq B_{\tildh}.
\end{align*}
Finally let $\norm{V}_\calV := \sup_{x \in S} \norm{\bmat{V(x) \\ \nabla V(x)}}_2$ denote the supremum norm on the space $\calV$.

Borrowing from the key insight in \cite{boffi2020learning}, 
we observe that any feasible solution  $\tilde{V}_n$ to \eqref{eq: empirical advrobust feasibility} achieves zero empirical risk on the loss $\tilde{\ell}_n(V) := \frac{1}{n} \sum_{i=1}^n \bone \curly{\tilde{h}(\xi_i, V) > -\tau}$.
Therefore, results from statistical learning theory regarding zero empirical risk minimizers can be applied to get fast rates for the generalization error. To do so, we define the adversarial loss class
$\tilde{\calH} := \curly{\tilde{h}(\cdot, V),\; V \in \calV}$.
Lemma 4.1 from \cite{boffi2020learning}, which is in turn adapted from Theorem 5 of \cite{srebro2010smoothness}, immediately gives the following bound on the generalization error.
    
\begin{lemma}[Generalization error bound]\label{lem: advrobust gen bound}
    Fix a $\delta \in (0,1)$. Let us assume Assumptions~\ref{assumption: bounded adversary traj} and \ref{assumption: bounded constants}. Suppose that the optimization problem \eqref{eq: empirical advrobust feasibility} is feasible and $\tilde{V}_n$ is a solution. Then the following holds with probability at least $1-\delta$ over $\xi_1,\dots,\xi_n$ drawn i.i.d.\ from $\calD$:
    \begin{align}\label{eq: adv gen bound}
        \err\paren{\tilde{V}_n} \leq K \paren{\frac{\log^3(n)}{\tau^2} \calR_n^2(\tilde{\calH}) + \frac{ \log\paren{\log\paren{B_{\tildh}/\tau}/\delta}}{n}},
    \end{align}
    where $K > 0$ is a universal constant and 
    \[\calR_n(\tilde\calH) := \sup_{\xi_1,\dots, \xi_n \in \calX} \Ex_{\sigma \sim \mathrm{Unif}\curly{\pm 1}^n}\brac{ \sup_{\tilde{h}(\cdot, V) \in \tilde\calH} \frac{1}{n} \abs{\sum_{i=1}^n \sigma_i \tildh(\xi_i, V)}}\]
    is the Rademacher complexity of the adversarial loss class $\tilde\calH$.
\end{lemma}
Lemma~\ref{lem: advrobust gen bound} reduces bounding the generalization error of an adversarially robust stability certificate to bounding the Rademacher complexity of the adversarial loss class $\tilde{\calH}$.  We  note that the nominal results of \citet[Lemma 4.1]{boffi2020learning} are recovered by setting the perturbation budget $\varepsilon=0$.

\subsection{A Simple Adversary-Agnostic Rademacher Complexity Bound}

A standard technique for controlling the Rademacher complexity
$\calR(\tilde{\calH})$ is appealing to 
Dudley's entropy integral \citep[Ch 5.]{wainwright19book}.
Specifically, if we show that for some $L_{\tildh}$,
\[
\abs{\tildh(\xi, V_1) - \tildh(\xi, V_2)} \leq L_{\tildh} \norm{V_1 - V_2}_{\calV} \:\: \forall \xi \in \calX,  \:\: V_1,V_2 \in \calV,
\]
then Dudley's inequality implies the bound
$\calR_n(\tilde{\calH}) \leq \frac{24 L_{\tildh} }{\sqrt{n}} \int_0^\infty \sqrt{\log N(\varepsilon; \calV, \norm{\cdot}_{\calV}) } \,d\varepsilon$.
%
%
Our first result shows that our main assumptions
are sufficient to ensure that $L_{\tilde{h}}$ can be controlled
with a uniform boundedness assumption on the adversary.
\begin{lemma}[Uniformly bounded adversaries are sufficient]\label{lem: bound on L_tildh}
Suppose that (i) Assumptions~\ref{assumption: bounded adversary traj} and \ref{assumption: bounded constants} hold, (ii) $B_{\delta} := \sup_{x \in S} \sup_{t \in \calT} \norm{\delta(t,x)}_2$ is finite, and (iii) the flow $\tilde \varphi_t(\xi)$ is unique and complete over $\calT$ for all $\xi\in\calX$ and all admissible $\delta(x,t)$. Let $L_h$ denote any constant such that
    $\abs{h(\xi, V_1) - h(\xi, V_2)} \leq L_h \norm{V_1-V_2}_{\calV}$
    for all $\xi \in X$ and $V_1,V_2 \in \calV$.
    Then, $L_{\tildh} \leq L_h + B_\delta$.
\end{lemma}
Lemma~\ref{lem: bound on L_tildh} shows that if the nominal system is input-to-state stable and if the adversary is uniformly bounded over the set $S$ from  Assumption~\ref{assumption: bounded adversary traj}, then by Dudley's inequality, the Rademacher complexity $\calR_n(\tilde{\calH})$ is on the same order as the nominal complexity $\calR_n(\calH)$. Consequently by Lemma~\ref{lem: advrobust gen bound}, the adversarial generalization bound $\err(\tilde{V}_n)$ is on the same order as the nominal bound $\err(\hat{V}_n)$. 
We show next that with stronger assumptions on the stability of the dynamics, we can obtain bounds on $\calR_n(\tilde{\calH})$ that are additive, rather than muliplicative, with respect to the nominal complexity $\calR_n(\calH)$. Furthermore, these bounds are also robust to Lipschitz adversarial perturbations.

\subsection{Improving the Adversarial Rademacher Complexity via Stability}\label{ssec: bounding rad complexity}

To improve the bound from Lemma~\ref{lem: bound on L_tildh}, we first adapt a fundamental fact from the calculus of Rademacher complexities \citep[Thm.~12, Property~5]{bartlett2002rademacher}, along with the trivial identity $\tildh(\cdot, V) = h(\cdot, V) + \paren{\tildh(\cdot, V) - h(\cdot, V)}$ to conclude that:
\begin{align}\label{eq: rad complex tildh - h}
    \calR_n(\tilde{\calH}) \leq \calR_n(\calH) + \sup_{\xi \in X} \sup_{V \in \calV} \frac{1}{\sqrt{n}} \abs{\tildh(\xi, V) - h(\xi, V)}.
\end{align}
Therefore, in order to bound $\calR_n(\tilde{\calH})$, it suffices to uniformly bound $\tildh(\xi, V) - h(\xi, V)$ over $\xi \in X$ and $V \in \calV$. To do so, we introduce the notion of $(\beta, \rho, \gamma)$-exponential-incrementally-input-to-state stability \citep{angeli2002lyapunov, boffi2020regret}.
\begin{definition}[$(\beta, \rho, \gamma)$-E-$\delta$ISS]\label{def: cont time EISS}
Let $\beta, \rho, \gamma> 0$ be positive constants. A continuous-time dynamical system $\dot{x}=f(x,t)$ is $(\beta, \rho, \gamma)$-exponential-incrementally-input-to-state stable ($(\beta, \rho, \gamma)$-E-$\delta$ISS) if, for any pair of initial conditions $(x_0,y_0)$ and signal $u(t)$ -- which can depend causally on $x,y$ -- the trajectories $\dot{x}(t) = f(x(t))$ and $\dot{y}(t) = f(y(t)) + u(t)$ satisfy for all $t \geq 0$:
\begin{align*}
    \norm{x_t - y_t}_2 &\leq \beta \norm{x_0 - y_0}_2 e^{-\rho t} + \gamma \int_0^t e^{-\rho(t-s)} \norm{u_s}_2 \,ds.
\end{align*}
\end{definition}
In short, the dependence of the distance between two trajectories on the initial conditions shrinks exponentially with time (incremental stability), and is input-to-state stable with respect to the inputs entering $y_t$. This notion of stability is strongly related to notion of contraction \citep{lohmiller1998contraction},
as illustrated by the following lemma.
\begin{lemma}[Contraction implies E-$\delta$ISS]
\label{lemma:contraction_implies_ISS}
Let $M(x, t)$ denote a positive definite Riemannian metric
and $f(x, t)$ denote a continuous-time dynamical system.
Suppose both $M$ and $f$ are continuously differentable,
and that there are constants $0 < \mu \leq L < \infty$ and
$\lambda > 0$ such that for all $x \in \R^n$ and $t \in \R_{\geq 0}$,
the metric $M(x, t)$ satisfies
$\mu I \preccurlyeq M(x, t) \preccurlyeq L I$,
and the function $f(x, t)$ satisfies:
\begin{align*}
    \frac{\partial f}{\partial x}(x, t)\T M(x, t) + M(x, t) \frac{\partial f}{\partial x}(x, t) + \dot{M}(x, t) \preccurlyeq -2\lambda M(x, t).
\end{align*}
Then, the dynamical system $\dot{x} = f(x, t)$ is $(\sqrt{L/\mu}, \lambda, \sqrt{L/\mu}$)-E-$\delta$ISS.
\end{lemma}
Lemma~\ref{lemma:contraction_implies_ISS}
is the analogous result of
Proposition~5.3 of \cite{boffi2020regret} for continuous-time systems.
We note that this result originally appeared in \cite[Section 3.7, Remark (vii)]{lohmiller1998contraction} without explicit proof.
Leveraging $(\beta,\rho,\gamma)$-E-$\delta$ISS, we can derive a uniform bound on $|\tildh(\xi,V)-h(\xi,V)|$ that scales with the stability parameters of the underlying system, which combined with inequality \eqref{eq: rad complex tildh - h} yields the following bounds on $\calR_n(\tildcalH)$ for the tubes~\eqref{def: norm-bounded pert ball}-\eqref{def: combo of adv}.

\begin{theorem}[E-$\delta$ISS yields additive bounds]\label{thm: CT EdISS rad complexity bound}
Put $B_X := \sup_{\xi \in \calX} \norm{\xi}_2$, let Assumption \ref{assumption: bounded constants} hold, and assume that the nominal system $f(x)$ is $(\beta, \rho, \gamma)$-E-$\delta$ISS. Then for

\begin{itemize}[noitemsep,topsep=1pt,parsep=0pt,partopsep=0pt]
\item adversarial trajectories drawn from the norm-bounded tube $\Delta^u_\varepsilon(\xi)$ defined in~\eqref{def: norm-bounded pert ball}, Assumption~\ref{assumption: bounded adversary traj} holds and
\begin{align}
    \calR_n(\tilde{\calH}) &\leq \calR_n(\calH) +  \brac{\paren{L_{\nabla V} + \eta L_V}\gamma\varepsilon\rho^{-1} + B_{\nabla V}\varepsilon + \nu}  \frac{1}{\sqrt{n}},
\end{align}
\item adversarial trajectories drawn from the Lipschitz tube $\Delta^x_\varepsilon(\xi)$ defined in~\eqref{def: lipschitz-bounded pert ball}, if $\varepsilon > 0$ is small enough such that $\gamma \varepsilon < \rho$, then Assumption~\ref{assumption: bounded adversary traj} holds and
\begin{align}
    \calR_n(\tilde{\calH}) &\leq \calR_n(\calH) +  \bigg[\paren{L_{\nabla V} + \eta L_V + B_{\nabla V}\varepsilon} \frac{\gamma\varepsilon\rho^{-1}}{1 - \gamma\varepsilon\rho^{-1}}e^{-1}B_X \beta\varepsilon \nonumber \\
    &\qquad\qquad\qquad + B_{\nabla V}B_X\beta\varepsilon + \nu \bigg] \frac{1}{\sqrt{n}},
\end{align}
\item adversarial trajectories drawn from the combined tube $\Delta^{x,u}_{\varepsilon_x, \varepsilon_u}$ defined in~\eqref{def: combo of adv}, if $\varepsilon_x > 0$ is small enough such that $\gamma \varepsilon_x < \rho$, then Assumption~\ref{assumption: bounded adversary traj} holds and
\begin{align}
    \calR_n(\tilde{\calH}) &\leq \calR_n(\calH) +  \bigg[\paren{L_{\nabla V} + \eta L_V + B_{\nabla V}\varepsilon_{x}} \frac{\gamma\varepsilon_u \rho^{-1} +  \gamma \varepsilon_x\rho^{-1}e^{-1}B_X \beta\varepsilon_x  }{1 - \gamma\varepsilon_x\rho^{-1}} \nonumber \\
    &\qquad\qquad\qquad + B_{\nabla V} \beta\varepsilon_x B_X  + B_{\nabla V}\varepsilon_u + \nu\bigg] \frac{1}{\sqrt{n}}.
\end{align}
\end{itemize}
\end{theorem}

In particular, Theorem \ref{thm: CT EdISS rad complexity bound} shows that $\calR_n(\tildcalH) \leq \calR_n(\calH) + O(1)\frac{1}{\sqrt{n}}$
for all the aforementioned adversary classes.  Here, $O(1)$ suppresses all problem specific constants.  This demonstrates that under the assumptions of Theorem~\ref{thm: CT EdISS rad complexity bound}, the Rademacher complexity of the resulting adversarial loss class is no more than an additive factor of order $O(1/\sqrt{n})$ greater than the Rademacher complexity class of the nominal loss class. 
Because a typical scaling of $\calR_n(\calH) \asymp \sqrt{k/n}$ where $k$ is the effective degrees of freedom of $\calV$, the
$O(1/\sqrt{n})$ term is often negligible compared to $\calR_n(\calH)$. 



The bounds in Theorem~\ref{thm: CT EdISS rad complexity bound} involving the Lipschitz adversary are only valid when the denominator $1 - \gamma \varepsilon \rho^{-1}$ is positive, hence the necessary assumption that $\gamma \varepsilon < \rho$. This is a necessary assumption; when the budget for the Lipschitz adversary is too large, then an adversary can cause the system to diverge exponentially. To illustrate this, consider the scalar system $\dotx = -\rho x$, which we can verify is $(1,\rho,1)$-E-$\delta$ISS, perturbed by a $\varepsilon$-Lipschitz adversary that adds $\varepsilon x$ to the dynamics such that $\doty = -(\rho - \varepsilon) y$. If $\varepsilon > \rho$, then the perturbed trajectory will diverge away from $0$ exponentially and we cannot hope to find a uniform bound on $\tildh(\xi, V) - h(\xi, V)$ for all $t$.

We conclude this section with an important example of a certificate function class and its associated Rademacher complexities.  This example further highlights that the additive $O(1/\sqrt{n})$ factor is comparatively negligible for many certificate function classes of interest.

\begin{example}[Lipschitz Parametric Function Classes] \label{ex: lip param fxn class}
Consider the parametric function class
\begin{equation}\label{def: lip param fxn class}
    \calV = \curly{V_{\theta}(\cdot) = g(\cdot, \theta): \theta \in \R^k,\;\norm{\theta} \leq B_\theta},
\end{equation}
where we assume $g: \R^p \times \R^k \to \R$ is twice-continuously differentiable. 
The description~\eqref{def: lip param fxn class} is very general; for example, feed-forward neural networks with differentiable activation functions and sum-of-squares polynomials lie in this function class. It is shown in \citet{boffi2020learning} 
that $\calR_n(\calH) = O(\sqrt{k/n})$. Combining this with Theorem~\ref{thm: CT EdISS rad complexity bound}, we conclude that
\[
\calR_n(\tildcalH) \leq \calR_n(\calH) + O(1/\sqrt{n}) = O(\sqrt{k/n}).
\]
\end{example}




\section{Learning Adversarially Robust Certificates in Practice}\label{sec: certs in practice}

In this section, we illustrate the practicality and effectiveness
of learning adversarial certificates. We consider the damped pendulum
with dynamics $m\ell^2 \ddot{\theta} + b\dot{\theta} + mg\ell\sin(\theta) = 0$, where we set $m=1$, $\ell = 1$, $b = 2$, and $g = 9.81$. The state space is given by $x = (\theta,  \dot{\theta}) \in \R^2$ with stable equilibrium is at the origin and we wrap $\theta$ to the interval $(-\pi, \pi]$. Consider the following certificate function class
\begin{align}\label{def: experiments NN certificate}
    \calV &= \curly{V_\theta(x) = x^\top \paren{L_\theta(x)^\top L_\theta(x) + I} x,\;\theta\in \R^{p \times h \times h \times p\cdot (2p)} }, 
\end{align}
where $L_\theta(x) \in \R^{2p \times p}$ is the re-shaped output of a fully-connected neural network with 2 hidden layers of width $h = 20$ and $\tanh$ activations. 

We first demonstrate the robustness properties of an adversarially trained Lyapunov function versus a nominal one. We collect $n = 1000$ trajectories with randomly sampled initial conditions $\xi \sim \mathrm{Unif}\paren{[-2, 2]^2}$. Each trajectory is rolled out using  \texttt{scipy.integrate.solve\_ivp} with horizon $T = 8$ and $dt = 0.05$, such the size of the total dataset is $1000 \times 160 \times 2$. Following \citet{boffi2020learning}, the nominal Lyapunov function $V_{\mathrm{nom}}$ is learned by minimizing the surrogate loss
\begin{equation}
    L(\theta; \eta, \lambda) = \sum_{i=1}^{1000} \sum_{k=1}^{160} \mathrm{ReLU}\brac{\ip{\nabla V_\theta(x_i(k)), \dot{x}_i(k))} + \eta V_{\theta}(x_i(k))  } + \lambda \norm{\theta}_2^2,
\end{equation}
where we set the exponential rate $\eta = 0.4$ and regularization parameter $\lambda = 0.1$. The loss is minimized for $500$ epochs with Adam \citep{kingma2014adam} with cosine decay, initialized at step size $0.005$, and batch size $1000$. 

Solving for the adversarially robust Lyapunov function is challenging due to the inner maximization problem over perturbations entering through the dynamics. 
%
As is standard in the adversarial learning literature, we instead approximate the true adversarially robust loss function via an alternating scheme, summarized in Algorithm~\ref{alg: training Vadv}. We set $m=5$, and each inner minimization of $L(\theta; \eta, \lambda)$ runs for $100$ epochs. 
%
The approximate adversarial computation uses a simple greedy heuristic: at any $x$, the maximal direction to increase the Lyapunov decrease condition $\ip{\nabla V(x), f(x) + \delta} + \eta V(x)$ is $\delta = c \nabla V(x)$, where $c>0$ is a normalizing factor to adjust $\delta$ for the adversarial budget $\varepsilon$. In this experiment, we use the Lipschitz adversary, and thus $c = \varepsilon \frac{\norm{x}_2}{\norm{\nabla V(x)}_2}$. Through Algorithm~\ref{alg: training Vadv}, we get an adversarially trained Lyapunov function $V_{\mathrm{adv}}$, which can only be less robust than the true adversarially robust function $\tilde{V}_n$ due to the suboptimal adversary computation. Nevertheless, our approximate robust Lyapunov function $V_{\mathrm{adv}}$ is seen to perform well in the face of practically relevant perturbations to the system.

\begin{algorithm}
\caption{Training adversarially robust Lyapunov function $\Vadv$ (Lipschitz adversary)}
\label{alg: training Vadv}

\begin{algorithmic}[1]
\State \textbf{Input:} Initial conditions $\curly{\xi_i}_{i=1}^{n}$, rate $\eta > 0$, adversarial budget $\varepsilon > 0$, alternations $m$. 
\State Compute nominal trajectories $\bfT = \curly{x(\xi_i)}_{i =1}^{n}$.\;
\For{$i=1, ..., m - 1$}
    \State Minimize $L(\theta; \eta, \lambda)$ with respect to $\bfT$ to get $V$.\;
    \State Re-compute $\bfT$ using dynamics $\dotx_i(t) = f(x_i(t)) + \varepsilon \frac{\norm{x_i(t)}_2}{\norm{\nabla V(x_i(t))}_2} \nabla V(x_i(t))$, $x_i(0) = \xi_i$.\;
\EndFor
\State Minimize $L(\theta; \eta, \lambda)$ with respect to $\bfT$ to get $V$.\;
\State \textbf{Output:} Adversarially trained Lyapunov function $\Vadv = V$.
\end{algorithmic}

\end{algorithm}

We assess the robustness of the nominal and robust certificates $\Vnom$ and $\Vadv$ by measuring how well they certify stability on various classes of perturbed trajectories. We first draw an additional test set of $n = 1000$ initial conditions from $\mathrm{Unif}\paren{[-2, 2]^2}$. For each class of perturbation, we vary the decrease rate parameter $\eta \in [0,1]$ (recall that the certificates $\Vnom$ and $\Vadv$ were trained with decrease rate $\eta = 0.4$) and measure both the proportion of whole trajectories as well as the total proportion of the $1000 \times 160$ states that satisfy the Lyapunov decrease condition with rate $\eta$. 

\setlength\intextsep{0pt}
\begin{wrapfigure}[12]{r}{0.5\textwidth}
    \centering
    \includegraphics[width=0.5\textwidth]{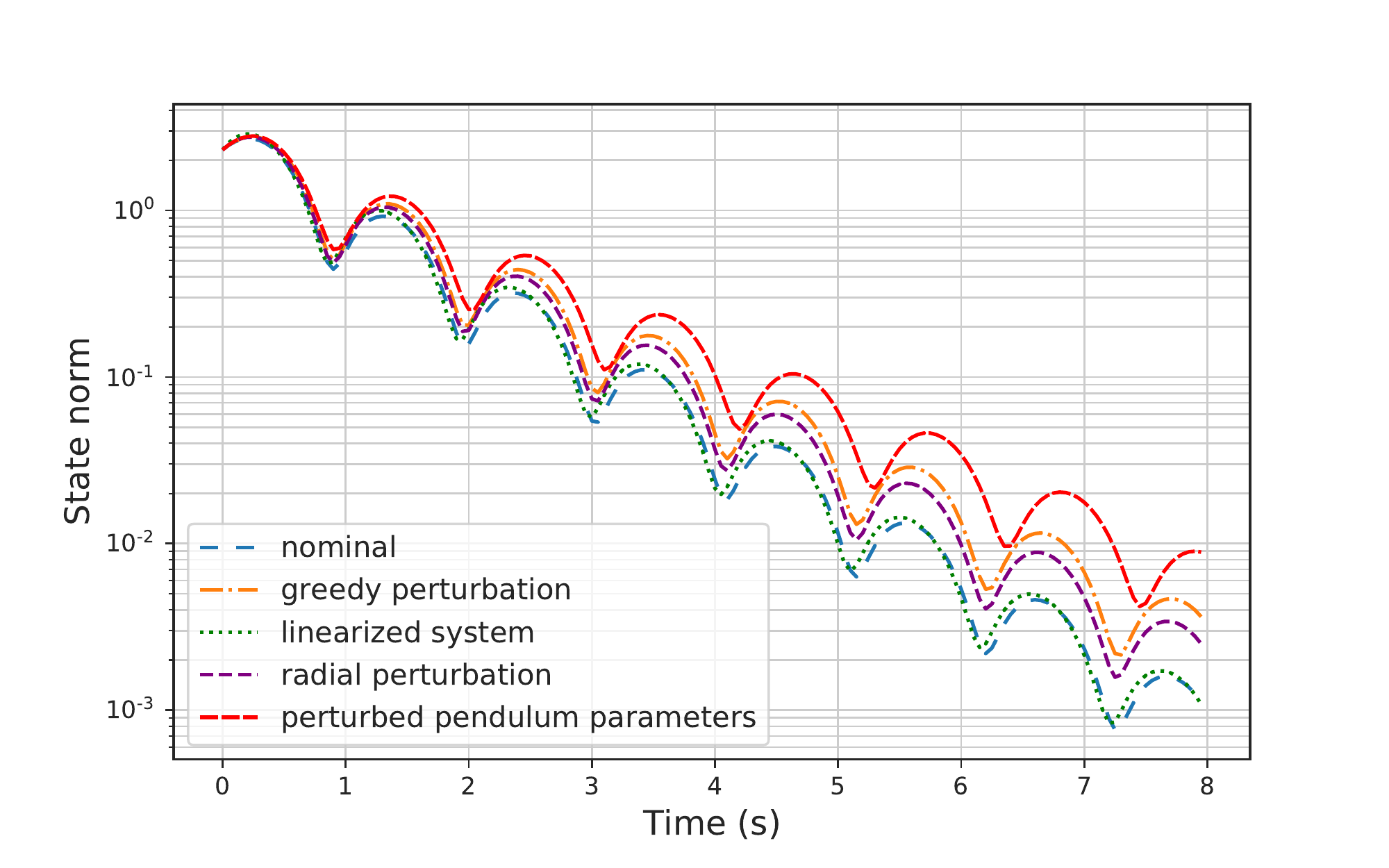}
    \caption{Norm of pendulum state over time starting at a fixed initial condition, for nominal and perturbed trajectories perturbed as described in Section~\ref{sec: certs in practice}.}
    \label{fig: perturbed trajectories}
\end{wrapfigure}

We consider the following four classes of perturbed trajectories:
\begin{enumerate}[itemsep=0pt]
    \item $\dotx = f(x) + \varepsilon \frac{\norm{x}_2}{\norm{\nabla \Vadv(x)}_2} \nabla \Vadv(x)$, analogous to the adversarial training process,
    \item  $\dotx = f(x) + \varepsilon x$, which is a Lipschitz adversary that aims to greedily maximize $\norm{x(t)}_2^2$ at any given time $t$,
    \item the dynamics resulting from using the linearization of the damped pendulum at the origin to generate the trajectories, and
    \item the dynamics resulting from setting $\tilde{m} = \tilde{\ell} = 1.1$ instead of $m = \ell = 1$. 
\end{enumerate}

The perturbation class 1 acts in the direction $\nabla \Vadv$, and thus the perturbed trajectories are tuned to degrade the performance of $\Vadv$. 
Additionally, the perturbation classes 3 and 4 can be viewed as instances of the sim-to-real gap, where there are model discrepancies between training and test.


Figure~\ref{fig: satisfaction rates diff perturbs}
plots the resulting Lyapunov decrease satisfaction rates for each type of perturbation. 
We observe that for each type of perturbation, the nominal certificate $\Vnom$ fails to certify any trajectories when $\eta = 0.4$. In contrast, the robust certificate $\Vadv$ certifies all trajectories for decrease rates $\eta = 0.4$.  We further observe that the robust certificate is also able to certify \emph{faster} decrease rates as well. 
Finally, we note that the trajectories resulting from perturbed pendulum parameters (perturbation class 4) actually cause the system to be more unstable than the greedy perturbations (perturbation class 1) used during training (see Figure~\ref{fig: perturbed trajectories}). Nevertheless, the robust certificate $\Vadv$ is able to certify stability for a large range of $\eta$.



\section{Conclusion}\label{sec: conclusion}
Motivated by bridging the sim-to-real gap, we proposed and analyzed an approach to learning adversarially robust Lyapunov certificates.  We showed that for systems that enjoy exponential incremental input-to-state stability, stability certificate functions that are robust to norm-bounded and Lipschitz adversarial perturbations to the system dynamics can be learned with negligible statistical overhead as compared to the nominal case. Future research directions include exploring the statistical tradeoffs occurring from progressively weaker notions of stability (e.g., incremental gain stability as defined in \citet{tu2021sample}), providing approximation guarantees for the adversarial training algorithm proposed in Section~\ref{sec: certs in practice}, and extending our results to provide statistical guarantees for policies synthesized from robust certificate functions~\citep{lindemann21rocbf,taylor21robustsynthesis}.

\begin{figure}
    \centering
    \stackunder[0pt]{\includegraphics[width=0.48\textwidth]{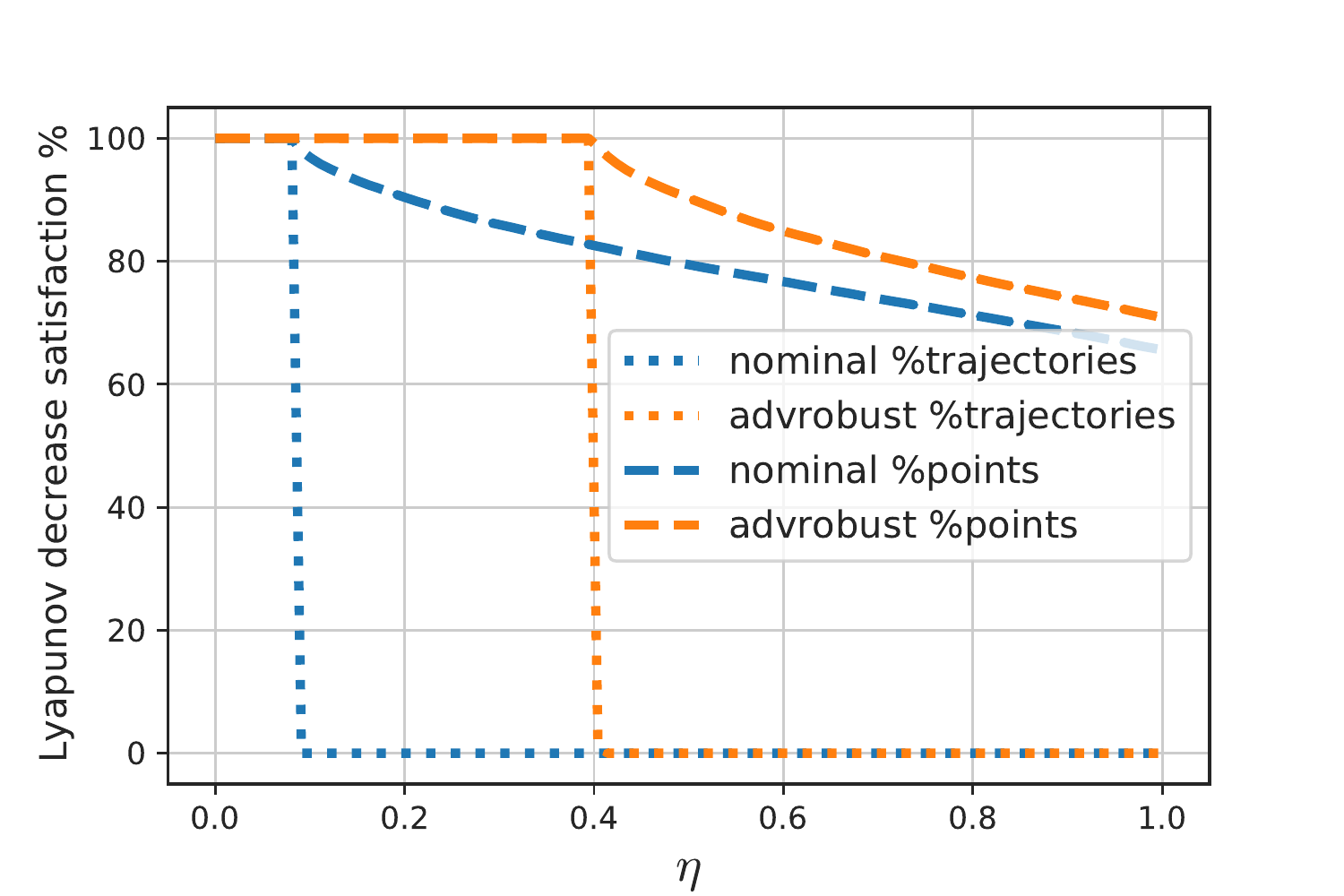}}{}
    \stackunder[0pt]{\includegraphics[width=0.48\textwidth]{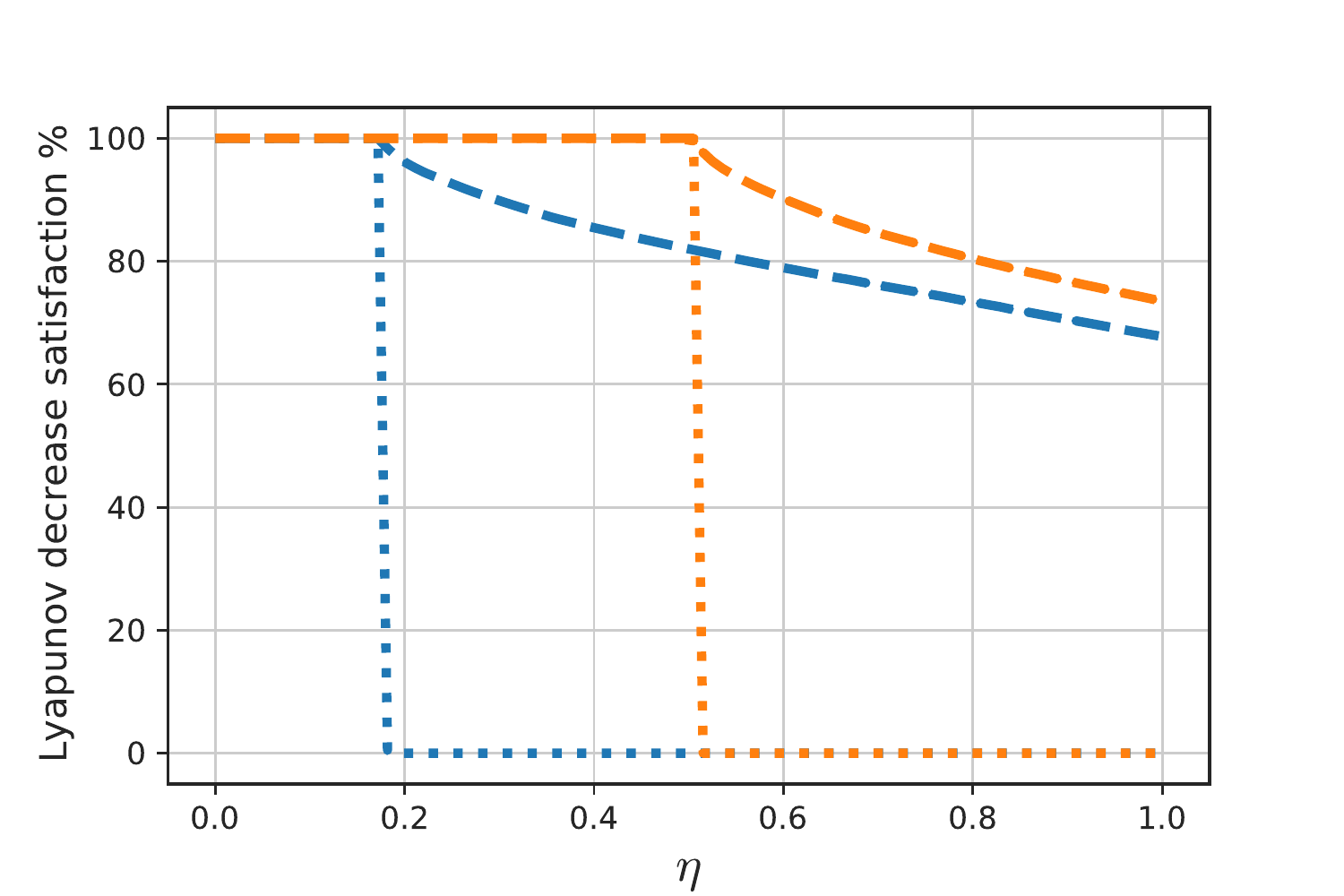}}{}
    \stackunder[0pt]{\includegraphics[width=0.48\textwidth]{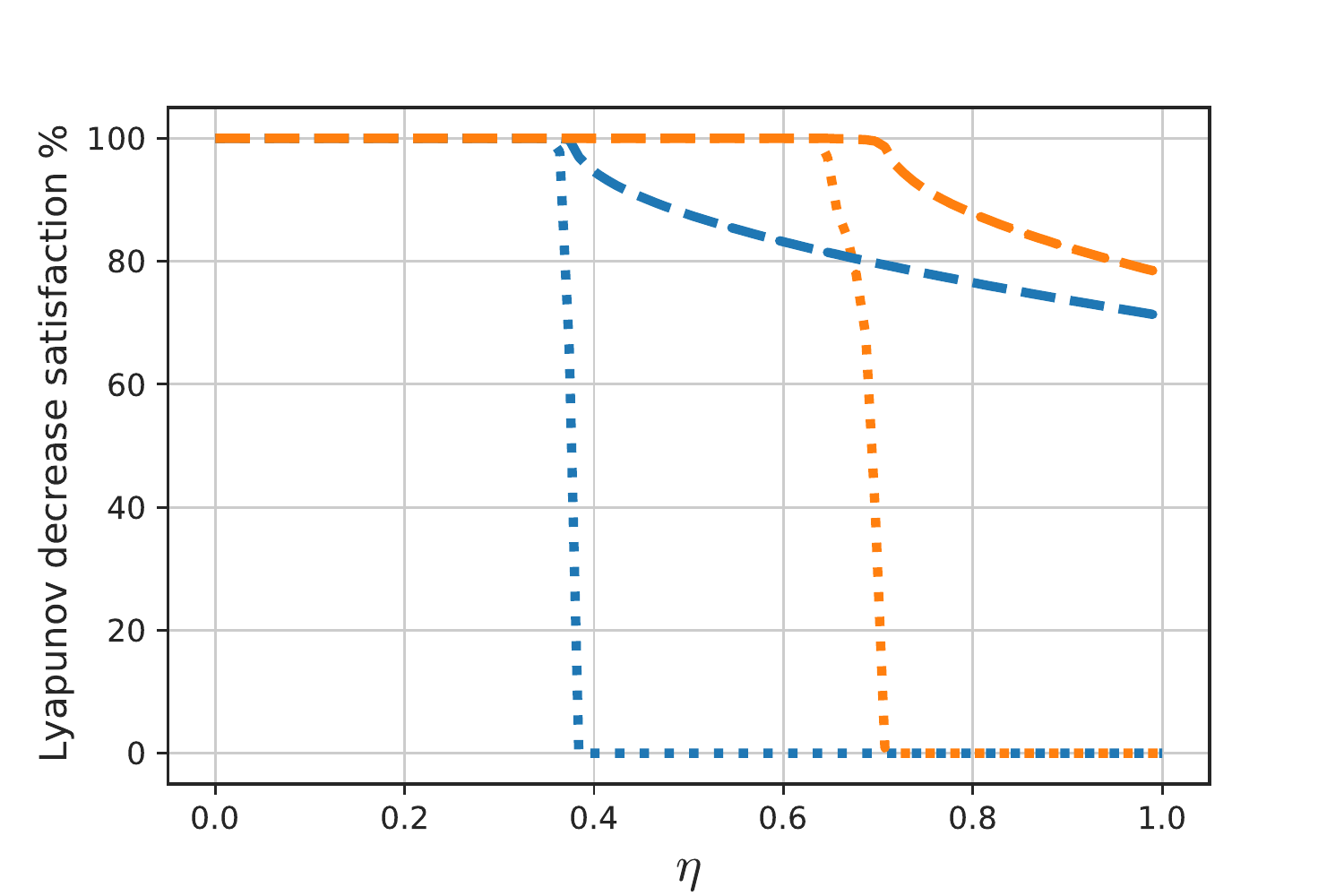}}{}
    \stackunder[0pt]{\includegraphics[width=0.48\textwidth]{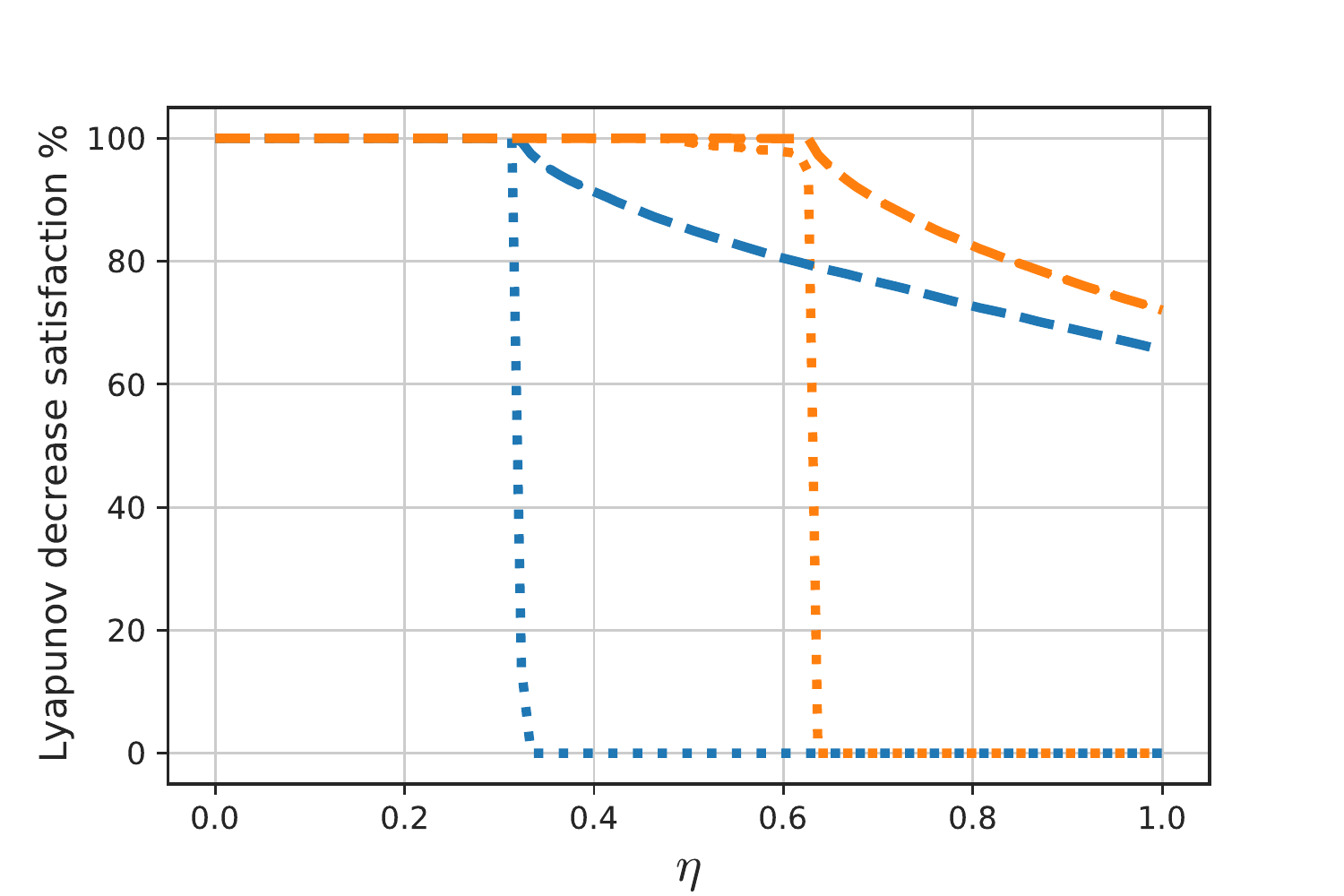}}{}
    \caption{
    Satisfaction rate of the Lyapunov decrease condition versus the exponential rate parameter $\eta$ of nominal and adversarially trained certificates $\Vnom$ and $\Vadv$ for four classes of perturbed trajectories. The percentage of trajectories and of total points satisfying the Lyapunov decrease condition for $\Vnom$ and $\Vadv$ are shown. Both $\Vnom$ and $\Vadv$ were trained with $\eta = 0.4$. Trajectories were generated by rolling out a fixed set of $1000$ initial conditions sampled from $\mathrm{Unif}\paren{[-2, 2]^2}$. \textbf{Upper left}: dynamics generated from gradient ascent on the adversarial certificate $\Vadv$, $\dotx = f(x) + \varepsilon\frac{\norm{x}}{\norm{\nabla{\Vadv}}}\nabla \Vadv$. \textbf{Upper right}: dynamics generated from a radial perturbation, $\dotx = f(x) + \varepsilon x$. \textbf{Lower left}: dynamics generated from system linearized at origin, $\dotx = J_{(0,0)} x$. \textbf{Lower right}: dynamics generated from perturbing the pendulum parameters, $\tilde{m} = 1.1$, $\tilde{\ell} = 1.1$.}
    \label{fig: satisfaction rates diff perturbs}
\end{figure}


\section*{Acknowledgements}

The authors thank Alexander Robey and Bruce D.\ Lee for various helpful discussions. Nikolai Matni is funded by NSF awards CPS-2038873, CAREER award ECCS-2045834, and a Google Research Scholar award.

\bibliographystyle{abbrvnat}
\bibliography{refs, corl_stability}

\clearpage
\appendix

\section{Proofs for Section \ref{sec: sample complexity}}

\subsection{Proof of Lemma \ref{lem: bound on L_tildh}}

The result follows by observing
\begin{align*}
    \tildh(\xi, V_1) - \tildh(\xi, V_2) &:= \sup_{\tvarphi \in \Delta_\varepsilon} \sup_{t \in \calT} \ip{\nabla V_1(\tvarphi_t(\xi)), f(\tvarphi_t(\xi)) + \delta_t} + \eta V_1(\tvarphi_t(\xi)) - \nu \\
    &\quad- \sup_{\tvarphi \in \Delta_\varepsilon} \paren{\ip{\nabla V_2(\tvarphi_t(\xi)), f(\tvarphi_t(\xi)) + \delta_t} + \eta V_2(\tvarphi_t(\xi)) - \nu} \\
    &\leq \sup_{\tvarphi \in \Delta_\varepsilon} \sup_{t \in \calT} \ip{\nabla V_1(\tvarphi_t(\xi)), f(\tvarphi_t(\xi)) + \delta_t} + \eta V_1(\tvarphi_t(\xi)) \\
    &\quad- \paren{\ip{\nabla V_2(\tvarphi_t(\xi)), f(\tvarphi_t(\xi)) + \delta_t} + \eta V_2(\tvarphi_t(\xi)) } \\
    &\leq L_h\norm{V_1 - V_2}_{\calV} + \ip{\nabla V_1(\tvarphi_t(\xi)) - \nabla V_2(\tvarphi_t(\xi)), \delta_t} \\
    &\leq L_h\norm{V_1 - V_2}_{\calV} + \norm{\nabla V_1(\tvarphi_t(\xi)) - \nabla V_2(\tvarphi_t(\xi))}_2\norm{\delta_t}_2 \\
    &\leq \paren{L_h + B_\delta} \norm{V_1 - V_2}_{\calV}. 
\end{align*}
Swapping the roles of $V_1$ and $V_2$ completes the proof.

\subsection{Proof of Lemma \ref{lemma:contraction_implies_ISS}}

We first state a few definitions.
Let $\mathsf{Sym}_{\geq 0}^{n \times n}$ denote the
space of $n \times n$ real-valued positive semi-definite matrices.
Given a Riemannian metric $M : \R^n \rightarrow \mathsf{Sym}_{\geq 0}^{n \times n}$, the geodesic distance associated with $M$ is:
\begin{align*}
    d_M(x, y) := \inf_{ \gamma \in \Gamma(x, y)} \int_0^1 \sqrt{\gamma'(s)\T M(\gamma(s)) \gamma'(s)} ds,
\end{align*}
where $\Gamma(x, y)$ denotes the set of smooth curves $\gamma$ with endpoints at $\gamma(0) = x$ and $\gamma(1) = y$.

Now, given a time-varying metric $M : \R^n \times \R \rightarrow \mathsf{Sym}_{\geq 0}^{n \times n}$,
a function $f(x, t)$ is said to be contracting in the metric
$M(x, t)$ at rate $\lambda$ if for all $x$ and $t$:
\begin{align*}
    \frac{\partial f}{\partial x}(x, t)\T M(x, t) + M(x, t) \frac{\partial f}{\partial x}(x, t) + \dot{M}(x, t) \preccurlyeq -2\lambda M(x, t).
\end{align*}

\begin{proposition}
\label{prop:contr_rob_cont}
Consider two nonlinear systems
\begin{align*}
    \dot{x}_p &= f(x_p, t) + d(x_p, t),\\
    \dot{x} &= f(x, t),
\end{align*}
where $f(x, t)$ is contracting in the metric $M(x, t)$. Then the geodesic distance $d_{M(\cdot, t)}(x_p(t), x(t))$ satisfies
the differential inequality
\begin{align*}
    \frac{d}{dt}d_{M(\cdot, t)}(x_p(t), x(t)) \leq - \lambda d_{M(\cdot, t)}(x_p(t), x(t)) + \norm{\Theta(x_p(t), t) d(x_p(t), t)}_2,
\end{align*}
where $M(x, t) = \Theta(x, t)\T \Theta(x, t)$.
\end{proposition}
\begin{proof}
Consider a geodesic $\gamma_t(s) : [0, 1] \rightarrow \R^n$ from $x_p(t)$ to $x(t)$ so that $\gamma_t(0) = x_p(t)$ and $\gamma_t(1) = x(t)$, and let $\gamma_t'$ denote the derivative of $\gamma_t$ with respect to its argument. Observe that the Riemannian energy is $E(\gamma_t) = d_M(x_p(t), x(t))^2$, and hence $\dot{E}(\gamma_t) = 2 d_M(x_p(t), x(t)) \left(\frac{d}{dt}d_M(x_p(t), x(t))\right)$. From the formula for the first variation of the Riemannian energy for a minimizing geodesic, 
\begin{equation*}
    \dot{E}(\gamma_t) = 2\ip{\gamma_t'(s),\dot{\gamma}_t(s)}|_{s=0}^{s=1} + 2 \frac{\partial E}{\partial t},
\end{equation*}
where $\dot{\gamma}_t$ denotes the time derivative of $\gamma_t$ along the flow of $x_p(t)$, and $\ip{\cdot,\cdot}$ denotes the Riemannian inner product. Note that while the functional form of $\gamma_t(s)$ is unknown, the values of its time derivative at the endpoints are fixed to be $\dot{x}_p$ and $\dot{x}$ for $s=0$ and $s=1$, respectively. From this, we find that
\begin{align*}
    \dot{E}(\gamma_t) &= 2\ip{\gamma'_t(1),f(x)} - 2\ip{\gamma'_t(0),f(x_p) + d(x_p, t)} + 2 \frac{\partial E}{\partial t}\\
    &\leq - 2\lambda E(\gamma_t) - 2\ip{\gamma'_t(0),d(x_p, t)},
\end{align*}
where the inequality stems from contraction of the nominal dynamics $f(x, t)$. This relation then implies the decrease condition
\begin{align*}
    \frac{d}{dt}d_{M(\cdot, t)}(x_p(t), x(t)) &\leq -\lambda d_{M(\cdot, t)}(x_p(t), x(t)) - \frac{1}{d_{M(\cdot, t)}(x_p(t), x(t))}\gamma'_t(0)\T\Theta(x_p)\T\Theta(x_p) d(x_p, t),\\
    &\leq -\lambda d_{M(\cdot, t)}(x_p(t), x(t)) + \frac{\norm{\Theta(x_p, t)\gamma'_t(0)}_2}{d_{M(\cdot, t)}(x_p(t), x(t))}\norm{\Theta(x_p, t) d(x_p, t)}_2.
\end{align*}
To complete the proof, observe that geodesics have constant energy, so that $\norm{\Theta(x_p, t)\gamma'_t(0)}_2 = d_{M(\cdot, t)}(x_p(t), x(t))$.
\end{proof}

We can now prove Lemma~\ref{lemma:contraction_implies_ISS}.
By Proposition~\ref{prop:contr_rob_cont}
and the assumption that $M(x, t) \preccurlyeq L I$,
\begin{align*}
    \frac{d}{dt} d_{M(\cdot, t)}(x_p(t), x(t)) &\leq -\lambda d_{M(\cdot, t)}(x_p(t), x(t)) + \norm{\Theta(x(t), t) d(x(t), t)}_2 \\
    &\leq -\lambda d_{M(\cdot, t)}(x_p(t), x(t)) + \sqrt{L} \norm{d(x_p(t), t)}_2.
\end{align*}
By the comparison lemma,
\begin{align*}
    d_{M(\cdot, t)}(x_p(t), x(t)) \leq d_{M(\cdot, 0)}(x_p(0), x(0)) e^{-\lambda t} + \sqrt{L}\int_0^t e^{-\lambda (t-s)} \norm{d(x_p(s), s)}_2 ds.
\end{align*}
Next, by the assumption that $\mu I \preccurlyeq M(x, t) \preccurlyeq L I$,
it is not hard to see (see e.g.\ Proposition~D.2 of \citet{boffi2020regret}) that
for all $x,y,t$,
\begin{align*}
    \sqrt{\mu} \norm{x-y}_2 \leq d_{M(\cdot, t)}(x, y) \leq \sqrt{L} \norm{x-y}_2
\end{align*}
Combining these inequalities, we have:
\begin{align*}
    \sqrt{\mu}\norm{x_p(t) - x(t)}_2 \leq \sqrt{L} \norm{x_p(0) - x(0)}_2 e^{-\lambda t} +  \sqrt{L}\int_0^t e^{-\lambda (t-s)} \norm{d(x_p(s), s)}_2 ds.
\end{align*}
The claim now follows by dividing both sides by $\sqrt{\mu}$.

\subsection{Proof of Theorem \ref{thm: CT EdISS rad complexity bound}}

We observe that for an arbitrary perturbation tube $\Delta(\xi)$
\begin{align*}
    \tildh(\xi, V) - h(\xi, V) &= \sup_{\tvarphi \in \Delta(\xi)} \sup_{t \in \calT} \ip{\nabla V(\tvarphi_t(\xi)), f(\tvarphi_t(\xi)) + \delta_t} + \eta V(\tvarphi_t(\xi)) - \nu \\
    &\quad - \sup_{t \in \calT} \paren{\ip{\nabla V(\varphi_t(\xi)), f(\varphi_t(\xi))} + \eta V(\varphi_t(\xi)) } \\
    &\leq \sup_{\tvarphi \in \Delta(\xi)} \sup_{t \in \calT} \ip{\nabla V(\tvarphi_t(\xi)), f(\tvarphi_t(\xi)) + \delta_t} + \eta V(\tvarphi_t(\xi)) - \nu \\
    &\quad -  \paren{\ip{\nabla V(\varphi_t(\xi)), f(\varphi_t(\xi))} + \eta V(\varphi_t(\xi)) } \\
    &\leq \sup_{\tvarphi \in \Delta(\xi)} \sup_{t \in \calT} \;(L_{\nabla V} + \eta L_{V})\norm{\tvarphi_t(\xi) - \varphi_t(\xi)}_2 + B_{\nabla V} \norm{\delta_t}_2 - \nu,
\end{align*}
and thus
\begin{equation}\label{eq: CT tildh - h bound}
    \abs{\tildh(\xi, V) - h(\xi, V)} \leq \sup_{\tvarphi \in \Delta(\xi)} \sup_{t \in \calT} \;(L_{\nabla V} + \eta L_{V})\norm{\tvarphi_t(\xi) - \varphi_t(\xi)}_2 + B_{\nabla V} \norm{\delta_t}_2 + \nu
\end{equation}

\begin{itemize}
\item \textbf{Norm-bounded Adversary} $\Delta_{\varepsilon}^u$: from E$\delta$ISS, we know that the deviation can be bounded by:
\begin{align*}
    \norm{\varphi_t - \tvarphi_t}_2 &\leq \beta \norm{\varphi_0 - \tvarphi_0} e^{-\rho t} + \gamma \int_0^t e^{-\rho(t-s)} \norm{\delta_s}_2 \,ds \\
    &\leq \gamma \varepsilon \int_0^t e^{-\rho(t - s)} \;ds \\
    &\leq \gamma \varepsilon \rho^{-1}.
\end{align*}
Plugging this back into \eqref{eq: CT tildh - h bound}, we get
\begin{align*}
    \abs{\tildh(\xi, V) - h(\xi, V)} &\leq \paren{L_{\nabla V} + \eta L_V}\gamma\varepsilon\rho^{-1} + B_{\nabla V} \varepsilon + \nu.
\end{align*}
Applying \eqref{eq: rad complex tildh - h} yields the desired result.

\item \textbf{Lipschitz Adversary} $\Delta_{\varepsilon}^x$: from E$\delta$ISS, we can bound the deviation by:
\begin{align*}
    \norm{\varphi_t - \tvarphi_t}_2 &\leq \beta \norm{\varphi_0 - \tvarphi_0} e^{-\rho t} + \gamma \int_0^t e^{-\rho(t-s)} \norm{\delta(\tvarphi_s)}_2 \,ds \\
    &\leq \gamma \varepsilon \int_0^t e^{-\rho(t - s)}\norm{\tvarphi_s}_2 \;ds \\
    &\leq \gamma \varepsilon \int_0^t e^{-\rho(t - s)}\paren{\norm{\varphi_s - \tvarphi_s}_2 + \norm{\varphi_s}_2} \;ds \\ 
    &\leq  \gamma \varepsilon \int_0^t e^{-\rho(t - s)}\norm{\varphi_s - \tvarphi_s}_2 \;ds  + \gamma \varepsilon \int_0^t e^{-\rho(t - s)}\beta \norm{\xi}_2 e^{-\rho s} \;ds \\
    &= \gamma \varepsilon \int_0^t e^{-\rho(t - s)}\norm{\varphi_s - \tvarphi_s}_2 \;ds  + \gamma \varepsilon \beta \norm{\xi}_2 t e^{-\rho t}.
\end{align*}
Taking the supremum over $t$ on both sides, we have
\begin{align*}
    \sup_{t \in \calT} \norm{\varphi_t - \tvarphi_t}_2 &\leq \sup_{t \in \calT} \gamma\varepsilon \int_0^t e^{-\rho(t - s)}\norm{\varphi_s - \tvarphi_s}_2 \;ds  + \gamma \varepsilon \beta \norm{\xi}_2 t e^{-\rho t} \\
    &\leq \gamma\varepsilon \sup_{t \in \calT} \int_0^t e^{-\rho(t - s)}\sup_{s \in \calT} \norm{\varphi_s - \tvarphi_s}_2 \;ds  + \sup_{t \in \calT} \gamma \varepsilon \beta \norm{\xi}_2 t e^{-\rho t} \\
    &\leq \gamma\varepsilon \rho^{-1} \sup_{s \in \calT} \norm{\varphi_s - \tvarphi_s}_2 + \gamma \varepsilon \rho^{-1} \beta \norm{\xi}_2 e^{-1},
\end{align*}
where the second term in the last line comes from optimizing $\max_{t} t e^{-\rho t}$, which attains its maximum $\frac{1}{\rho e}$ at $t = 1/\rho$. Since by assumption, $\frac{\gamma \varepsilon}{\rho} < 1$, we have
\begin{align*}
    \sup_{t \in \calT} \norm{\varphi_t - \tvarphi_t}_2 &\leq \frac{\gamma\varepsilon\rho^{-1}}{1 - \gamma\varepsilon\rho^{-1}} \beta e^{-1} \norm{\xi}_2.
\end{align*}
Plugging this into \eqref{eq: CT tildh - h bound}, we get
\begin{align*}
    \abs{\tildh(\xi, V) - h(\xi, V)} &\leq  \sup_{\tvarphi \in \Delta_\varepsilon(\xi)} \sup_{t\in \calT}\; \paren{L_{\nabla V} + \eta L_V}\norm{\tvarphi_t(\xi) - \varphi_t(\xi)}_2 + B_{\nabla V} \norm{\delta(\tvarphi_t(\xi))}_2  + \nu\\
    &\leq \sup_{\tvarphi \in \Delta_\varepsilon(\xi)} \sup_{t\in\calT}\; \paren{L_{\nabla V} + \eta L_V}\norm{\tvarphi_t - \varphi_t}_2 +  B_{\nabla V} \varepsilon \paren{\norm{\varphi_t - \tvarphi_t}_2 + \norm{\tvarphi_t}_2}  + \nu\\
    &\leq \paren{L_{\nabla V} + \eta L_V + B_{\nabla V}\varepsilon} \frac{\gamma\varepsilon\rho^{-1}}{1 - \gamma\varepsilon\rho^{-1}} \beta e^{-1} \norm{\xi}_2 + \sup_{t} B_{\nabla V} \beta \varepsilon\norm{\xi}_2 e^{-\rho t}  + \nu\\
    &\leq \brac{\paren{L_{\nabla V} + \eta L_V + B_{\nabla V}\varepsilon} \frac{\gamma\varepsilon\rho^{-1}}{1 - \gamma\varepsilon\rho^{-1}}e^{-1} + B_{\nabla V}}\beta\varepsilon \norm{\xi}_2  + \nu.
\end{align*}
Applying \eqref{eq: rad complex tildh - h} yields the desired result.

\item \textbf{Combined Adversary} $\Delta_{\varepsilon_x, \varepsilon_u}^{x,u}$: proof follows similarly to the Lipschitz adversary case. Using E$\delta$ISS, we have
\begin{align*}
    \norm{\varphi_t - \tvarphi_t}_2 &\leq \beta \norm{\varphi_0 - \tvarphi_0} e^{-\rho t} + \gamma \int_0^t e^{-\rho(t-s)} \norm{\delta(\tvarphi_s) + \delta_t}_2 \,ds \\
    &\leq \gamma \varepsilon_x \int_0^t e^{-\rho(t - s)}\norm{\tvarphi_s}_2 \;ds + \gamma \varepsilon_u \rho^{-1} \\
    &= \gamma \varepsilon_x \int_0^t e^{-\rho(t - s)}\norm{\varphi_s - \tvarphi_s}_2 \;ds  + \gamma \varepsilon_x \beta \norm{\xi}_2 t e^{-\rho t} +  \gamma \varepsilon_u \rho^{-1}.
\end{align*}
Since $\gamma\varepsilon_x < \rho$, we take the supremum of both sides and shift terms around to get
\begin{align*}
    \sup_{t \in \calT} \norm{\varphi_t - \tvarphi_t}_2 &\leq \frac{\gamma \varepsilon_u \rho^{-1} + \gamma\varepsilon_x\rho^{-1}\beta e^{-1} \norm{\xi}_2 }{1 - \gamma\varepsilon_x\rho^{-1}},
\end{align*}
and thus plugging into \eqref{eq: CT tildh - h bound}, we get
\begin{align*}
    \abs{\tildh(\xi, V) - h(\xi, V)}  &\leq \paren{L_{\nabla V} + \eta L_V + B_{\nabla V}\varepsilon_{x}} \frac{\gamma\varepsilon_u \rho^{-1} +  \gamma \varepsilon_x\rho^{-1}e^{-1}\beta\varepsilon_x \norm{\xi}_2 }{1 - \gamma\varepsilon_x\rho^{-1}} \\
    &\quad + B_{\nabla V} \beta\varepsilon_x \norm{\xi}_2  + B_{\nabla V}\varepsilon_u + \nu.
\end{align*}
Applying \eqref{eq: rad complex tildh - h} yields the desired result.

\end{itemize}

\section{Adversarially Robust Certificates in Discrete Time}\label{sec: DT results}

In the discrete time setting, we consider the system $x_{t+1} = f(x_t)$. Like in the continuous time case, $f$ is continuous and unknown, the state $x \in \R^p$ is fully observed, the initial conditions $\xi$ are drawn from some compact set $\calX$, and $\varphi_t(\xi)$ denotes the map to the state at time $t$ given initial condition $\xi$. Given a candidate Lyapunov function $V$, we define the nominal and adversarial (exponential) Lyapunov decrease conditions as:
\begin{align*}
    h(\xi, V) &= \max_{t \leq T} V\paren{\varphi_{t+1}(\xi)} - \eta^2 V\paren{\varphi_t(\xi)} \\
    \tilde{h}_\nu(\xi, V) &= \max_{\tvarphi \in \Delta_\varepsilon(\xi)} \max_{t \leq T} V\paren{\tvarphi_{t+1}(\xi)} - \eta^2 V\paren{\tvarphi_t(\xi)} - \nu,\quad \nu \geq 0,
\end{align*}
where $0 < \eta < 1$, as well as the corresponding loss classes $\calH$ and $\tilde{\calH}$. The stability certification problem can be posed as the following feasibility problem
\begin{equation}\label{eq: DT feasibility}
    \mathrm{Find}_{V \in \calV}\;\mathrm{s.t.} \;\; \tilde{h}_\nu(\xi, V) \leq -\tau,\quad i = 1,\dots,n.
\end{equation}
We make the following assumptions.
\begin{assumption}[Discrete-time stability in the sense of Lyapunov]\label{assumption: DT Stability isL} Fix a perturbation set $\Delta(\cdot)$. There exists a compact set $S \subseteq \R^p$ such that $\tvarphi_t(\xi) \in S$ for all $\xi \in \calX$, $t \leq T$, and $\tvarphi_t(\xi) \in \Delta(\xi)$.
\end{assumption}

\begin{assumption}[Regularity of $\calV$]\label{assumption: DT regularity of V}
    There exist a constant $L_V$ such that for every $V \in \calV$, the map $x \to V(x)$ over $x \in S$ is $L_V$-Lipschitz.
\end{assumption}
Under Assumptions~\ref{assumption: DT Stability isL} and~\ref{assumption: DT regularity of V}, and the continuity of the dynamics $f(x)$, there exist constants $B_V$ and $B_{\tildh}$ such that
\begin{align*}
    \sup_{V \in \calV} \sup_{x \in S} \abs{V(x)} \leq B_V,\;\sup_{V \in \calV} \sup_{\xi \in \calX} \abs{\tildh(\xi, V)} \leq B_{\tildh}.
\end{align*}
Finally, let $\norm{V}_\calV := \sup_{x \in S} \abs{V(x)}$ denote the supremum norm on the space $\calV$. Under Assumptions~\ref{assumption: DT Stability isL} and~\ref{assumption: DT regularity of V}, and the discrete time definition of $\tildh$, Lemma~\ref{lem: advrobust gen bound} holds in precisely the same form, such that we once again need only to bound the Rademacher complexity of the adversarial loss class. We now prove the discrete-time variant of the adversary-agnostic Rademacher complexity bound.

\begin{lemma}[Discrete-time analogue of Lemma~\ref{lem: bound on L_tildh}]
Suppose that Assumptions~\ref{assumption: DT Stability isL} and~\ref{assumption: DT regularity of V} hold. Let $L_h$ denote any constant such that
    $\abs{h(\xi, V_1) - h(\xi, V_2)} \leq L_h \norm{V_1-V_2}_{\calV}$
    for all $\xi \in X$ and $V_1,V_2 \in \calV$.
    Then, $L_{\tildh} \leq L_h + 2$.
\end{lemma}

\begin{proof}
Writing out $\tildh(\xi, V_1) - \tildh(\xi, V_2)$, for arbitrary $\xi$, $V_1$, $V_2$, we have:
\begin{align*}
    \tilde{h}(\xi, V_1) - \tilde{h}(\xi, V_2) &= \max_{\tvarphi \in \Delta(\xi)} \max_{t \leq T} V_1(\tvarphi_{t+1}) - \eta^2 V_1(\tvarphi_t) - \max_{\tvarphi  \in \Delta(\xi)} \max_{t \leq T} \paren{V_2(\tvarphi_{t+1}) - \eta^2 V_2(\tvarphi_t)} \\
    &\leq \max_{\tvarphi  \in \Delta(\xi)} \max_{t \leq T} V_1(\tvarphi_{t+1}) - \eta^2 V_1(\tvarphi_t) - \paren{V_2(\tvarphi_{t+1}) - \eta^2 V_2(\tvarphi_t)}.
\end{align*}
For any time $t$, we have
\begin{align*}
    V_1(\tvarphi_{t+1}) - \eta^2 V_1(\tvarphi_t) - \paren{V_2(\tvarphi_{t+1}) - \eta^2 V_2(\tvarphi_t)} &= V_1(f(\tvarphi_t)) - \eta^2 V_1(\tvarphi_t) - \paren{V_2(f(\tvarphi_t)) - \eta^2 V_2(\tvarphi_t)} \\
    &\quad + \paren{V_1(\tvarphi_{t+1}) - V_2(\tvarphi_{t+1})} \\
    &\quad - \paren{V_1(f(\tvarphi_t)) - V_2(f(\tvarphi_t))} \\
    &\leq \abs{h(\xi, V_1) - h(\xi, V_2)} + 2 \sup_{x \in S} \abs{V_1(x) - V_2(x)} \\
    &\leq L_h \norm{V_1 - V_2}_{\calV} + 2 \norm{V_1 - V_2}_{\calV},
\end{align*}
where added and subtracted $V_1(f(\tvarphi_t))$ and $V_2(f(\tvarphi_t))$, and used Assumption~\ref{assumption: bounded adversary traj} to bound the leftover terms using the definition of $\norm{\cdot}_\calV$. The argument is symmetric for $V_1, V_2$, so we have
\[
\abs{\tilde{h}(\xi, V_1) - \tilde{h}(\xi, V_2)} \leq (L_h + 2) \norm{V_1 - V_2}_{\calV},
\]
for all $\xi \in X$ and $V_1, V_2 \in \calV$.
\end{proof}

We note this is at first glance a better bound than the continuous-time version. This can be attributed to the fact that Assumption~\ref{assumption: DT Stability isL} is at face value more restrictive than its continuous-time analogue Assumption~\ref{assumption: bounded adversary traj}, since it implicitly enforces a norm-constraint on $\delta_t$ such that it cannot push $x_t$ out of the compact set $S$, which is a property independent of the certificate $V$. On the other hand, Assumption~\ref{assumption: bounded adversary traj} does not immediately enforce a norm-constraint on $\delta_t$--the implicit constraint depends on the choice of $V$, where $\delta_t$ cannot render optimization problem~\eqref{eq: empirical advrobust feasibility} infeasible. As Assumption~\ref{assumption: DT Stability isL} is in a sense more restrictive than Assumption~\ref{assumption: bounded adversary traj}, the bound we get is stronger.

We now re-introduce the norm-bounded, Lipschitz, and combined adversarial trajectory tubes in discrete time:
\begin{align}
    &\Delta^{u}_\varepsilon(\xi) := \curly{\tvarphi: \tilde{\varphi}_{t+1} = f(\tilde{\varphi}_t) + \delta_t,\; \tilde{\varphi}_0 = \xi, \; \norm{\delta_t}_2\leq \varepsilon} \label{def: DT norm-bounded adversary}\\
    &\Delta^x_\varepsilon(\xi) := \curly{\tvarphi: \tilde{\varphi}_{t+1} = f(\tilde{\varphi}_t) + \delta(\tilde{\varphi}_t),\; \tilde{\varphi}_0 = \xi, \; \norm{\delta(\tilde{\varphi}_t)}_2 \leq \varepsilon \norm{\tilde{\varphi}_t}_2} \label{def: DT lipschitz adversary} \\
    &\Delta^{x,u}_{\varepsilon_x,\varepsilon_u}(\xi) := \curly{\tvarphi: \tvarphi_{t+1} = f(\tvarphi_t) + \delta^x(\tvarphi_t) + \delta^u_t,\; \tvarphi_0 = \xi,\; \norm{\delta^x(\tvarphi_t)}_2 \leq \varepsilon_x \norm{\tvarphi_t}_2,\; \norm{\delta^u_t}_2 \leq \varepsilon_u }. \label{def: DT combined adversary}
\end{align}
Accordingly, we introduce $(\beta,\rho,\gamma)$-E$\delta$ISS in discrete time.
\begin{definition}[Discrete-time $(\beta, \rho, \gamma)$-E$\delta$ISS]
Let $\beta, \gamma > 0$ be positive constants and $\rho \in (0,1)$. A discrete-time dynamical system $f(x,t)$ is $(\beta, \rho, \gamma)$-exponential-incrementally-input-to-state stable if for every pair of initial conditions $(x_0,y_0)$ and signal $u_t$ (which can depend causally on $x,y$), the trajectories $x_{t+1} = f(x_t,t)$ and $y_{t+1} = f(y_t,t) + u_t$ satisfy for all $t \geq 0$:
\begin{equation}\label{def: DT EdISS}
    \norm{x_t - y_t}_2 \leq \beta \rho^t \norm{x_0 - y_0}_2 + \gamma \sum_{k=0}^{t-1}\rho^{t-1-k}\norm{u_k}_2.
\end{equation}

\end{definition}
As mentioned earlier in this paper, Proposition 5.3 of \citep{boffi2020regret} gives that a discrete-time system contracting with respect to some metric is $(\beta, \rho, \gamma)$-E$\delta$ISS. We are now ready to provide the discrete-time analogues to the adversarial Rademacher complexities provided in Theorem~\ref{thm: CT EdISS rad complexity bound}.
\begin{theorem}[Discrete-time analogue to Theorem~\ref{thm: CT EdISS rad complexity bound}]\label{thm: DT edISS rad complexity bound}
Put $B_X := \sup_{\xi \in \calX} \norm{\xi}_2$, let Assumption \ref{assumption: DT regularity of V} hold, and assume that the nominal discrete-time system $f(x)$ is $(\beta, \rho, \gamma)$-E$\delta$ISS. Then for
\begin{itemize}[noitemsep,topsep=1pt,parsep=0pt,partopsep=0pt]
\item adversarial trajectories drawn from the norm-bounded tube $\Delta^u_\varepsilon(\xi)$ defined in~\eqref{def: DT norm-bounded adversary}, Assumption~\ref{assumption: DT Stability isL} holds and
\begin{align}
    \calR_n(\tilde{\calH}) &\leq \calR_n(\calH) +  \brac{(1 + \eta^2)L_V\frac{\gamma \varepsilon}{1 - \rho}  + \nu}\frac{1}{\sqrt{n}},
\end{align}
\item adversarial trajectories drawn from the Lipschitz tube $\Delta^x_\varepsilon(\xi)$ defined in~\eqref{def: DT lipschitz adversary}, if $\varepsilon > 0$ is small enough such that $\rho + \gamma \varepsilon < 1$, then Assumption~\ref{assumption: DT Stability isL} holds and
\begin{align}
    \calR_n(\tilde{\calH}) &\leq \calR_n(\calH) + \brac{L_V \beta B_X (\rho + \gamma\varepsilon + \eta^2) + \nu} \frac{1}{\sqrt{n}},
\end{align}
\item adversarial trajectories drawn from the combined tube $\Delta^{x,u}_{\varepsilon_x, \varepsilon_u}$ defined in~\eqref{def: DT combined adversary}, if $\varepsilon_x > 0$ is small enough such that $\rho + \gamma \varepsilon_x < 1$, then Assumption~\ref{assumption: DT Stability isL} holds and
\begin{align}
    \calR_n(\tilde{\calH}) &\leq \calR_n(\calH) +  \bigg[(1 + \eta^2)L_V \paren{\frac{1 - \rho}{1 - (\rho + \gamma \varepsilon_x)} \frac{\gamma \beta B_X \varepsilon_x }{e \rho \log\paren{\rho^{-1}}}  + \frac{\gamma \varepsilon_u}{1 - \rho } }+ \nu\bigg] \frac{1}{\sqrt{n}}.
\end{align}
\end{itemize}
\end{theorem}
We note that the necessary condition that $\rho + \gamma \varepsilon < 1$ for the Lipschitz and combined adversaries is nicely analogous to the continuous-time case where we needed $\gamma\varepsilon < \rho$; in both cases, the adversary cannot be powerful enough to de-stabilize the system. One can consider the scalar system $x_{t+1} = \rho x_t$, $0 < \rho < 1$ and the adversary $\delta(x) = \varepsilon x$ to see why this condition cannot be loosened in general. We now provide the proof to Theorem \ref{thm: DT edISS rad complexity bound}.

\begin{proof}
We observe that for an arbitrary perturbation tube $\Delta(\xi)$
\begin{align*}
    \tildh(\xi, V) - h(\xi, V) &= \sup_{\tvarphi \in \Delta(\xi)} \sup_{t \leq T}\; V(\tvarphi_{t+1}) - \eta^2 V(\tvarphi_t)  - \nu  - \sup_{t \leq T} \paren{V(\varphi_{t+1}) - \eta^2 V(\varphi_t)} \\
    &\leq \sup_{\tvarphi \in \Delta(\xi)} \sup_{t \leq T}\; V(\tvarphi_{t+1}) - V(\varphi_{t+1}) - \eta^2 \paren{V(\tvarphi_{t}) - V(\varphi_t)} - \nu\\
    &\leq \sup_{\tvarphi \in \Delta(\xi)} \sup_{t \leq T}\; (1 + \eta^2)L_V \norm{\tvarphi_t - \varphi_t}_2  - \nu,
\end{align*}
and thus
\begin{equation}\label{eq: DT tildh - h bound}
    \abs{\tildh(\xi, V) - h(\xi, V)} \leq \sup_{\tvarphi \in \Delta(\xi)} \sup_{t \leq T}\; (1 + \eta^2)L_V \norm{\tvarphi_t - \varphi_t}_2  + \nu.
\end{equation}
Thus, it suffices to establish bounds on $\norm{\tvarphi_t - \varphi_t}_2$.

\begin{itemize}
\item \textbf{Norm-bounded Adversary} $\Delta_{\varepsilon}^u$: from E$\delta$ISS, we know that the deviation can be bounded by:
\begin{align*}
    \norm{\tvarphi_t - \varphi_t}_2 &\leq \beta \rho^{t} \norm{\tvarphi_0 - \varphi_0}_2  + \gamma \sum_{k=0}^{t-1} \rho^{t-1-k} \norm{\delta_k}_2 \\
    &\leq \gamma \sum_{k=0}^{t-1} \rho^{t-1-k} \varepsilon \\
    &\leq \frac{\gamma \varepsilon}{1 - \rho}.
\end{align*}
Plugging this back into~\eqref{eq: DT tildh - h bound}, we get
\begin{align*}
    \abs{\tildh(\xi, V) - h(\xi, V)} &\leq (1 + \eta^2)L_V\frac{\gamma \varepsilon}{1 - \rho}  + \nu.
\end{align*}

\item \textbf{Lipschitz Adversary} $\Delta_{\varepsilon}^x$: using a more careful analysis, we can get a finer bound in the discrete-time setting than the continuous-time setting, where we get an explicit time-dependent upper bound on the deviation $\norm{\tvarphi_t - \varphi_t}_2$. From E$\delta$ISS, observe that for any $k \leq T$,
\begin{align}
    \norm{\tvarphi_k - \varphi_k}_2 &\leq \gamma \sum_{i = 0}^{k-1} \rho^{k-1-i}\norm{\delta(\tvarphi_i)}_2 \nonumber \\
    &\leq \gamma\varepsilon \sum_{i = 0}^{k-1} \rho^{k-1-i} \norm{\tvarphi_i}_2 \nonumber \\
    &\leq \gamma\varepsilon \sum_{i = 0}^{k-1} \rho^{k-1-i} \paren{\norm{\tvarphi_i - \varphi_i}_2 + \norm{\varphi_i}_2} \nonumber \\
    &\leq \gamma\varepsilon \sum_{i = 0}^{k-1} \rho^{k-1-i} \norm{\varphi_i}_2 + \gamma\varepsilon \sum_{i = 0}^{k-1} \rho^{k-1-i} \norm{\tvarphi_i - \varphi_i}_2 \nonumber \\
    &\leq \gamma\varepsilon \sum_{i=0}^{k-1}\rho^{k-1} \beta\norm{\xi}_2  + \gamma\varepsilon \sum_{i = 1}^{k-1} \rho^{k-1-i} \norm{\tvarphi_i - \varphi_i}_2. \label{eq: EISS tvarphi_t varphi_t difference}
\end{align}
We keep the sum in the first term to keep the our later algebraic manipulations clear. We also observe we can move the starting index of the second sum from $0$ to $1$, since $\tvarphi_0 - \varphi_0 = 0$. Now fixing any timestep $t \leq T$ for $t \geq 2$, we apply \eqref{eq: EISS tvarphi_t varphi_t difference} recursively:
\begin{align*}
    \norm{\tvarphi_t - \varphi_t}_2 &\leq \gamma\varepsilon \sum_{k_1=0}^{t-1}\rho^{t-1} \beta\norm{\xi}_2  + \gamma\varepsilon \sum_{k_1 = 1}^{t-1} \rho^{t-1-k_1} \norm{\tvarphi_{k_1} - \varphi_{k_1}}_2 \\
    &\leq \gamma\varepsilon \sum_{k_1=0}^{t-1}\rho^{t-1} \beta\norm{\xi}_2  \\
    &\quad + \gamma\varepsilon \sum_{k_1 = 1}^{t-1} \rho^{t-1-k_1} \paren{ \gamma \varepsilon \sum_{k_2 = 0}^{k_1 - 1} \rho^{k_1 - 1} \beta \norm{\xi}_2 + \gamma \varepsilon \sum_{k_2 = 1}^{k_1 - 1} \rho^{k_1 - k_2 - 1} \norm{\tvarphi_{k_2} - \varphi_{k_2}}_2 } \\
    &\leq \gamma\varepsilon \sum_{k_1=0}^{t-1}\rho^{t-1} \beta\norm{\xi}_2  + \paren{\gamma\varepsilon}^2 \sum_{k_1 = 1}^{t-1} \sum_{k_2 = 0}^{k_1 - 1}\rho^{t-2} \beta \norm{\xi}_2 \\
    &\quad + \paren{\gamma \varepsilon}^2 \sum_{k_1 = 1}^{t-1} \sum_{k_2 = 1}^{k_1 - 1} \rho^{t - 2 - k_2} \norm{\tvarphi_{k_2} - \varphi_{k_2}}_2 \\
    &\leq \gamma\varepsilon \sum_{k_1=0}^{t-1}\rho^{t-1} \beta\norm{\xi}_2  + \paren{\gamma\varepsilon}^2 \sum_{k_1 = 1}^{t-1} \sum_{k_2 = 0}^{k_1 - 1}\rho^{t-2} \beta \norm{\xi}_2 + \cdots \\
    &\quad + (\gamma\varepsilon)^j \sum_{k_1 = 1}^{t-1} \sum_{k_2 = 1}^{k_1 - 1}\cdots \sum_{k_j = 0}^{k_{j-1} - 1} \rho^{t - j} \beta \norm{\xi}_2 \\
    &\quad + (\gamma\varepsilon)^j \sum_{k_1 = 1}^{t-1} \sum_{k_2 = 1}^{k_1 - 1}\cdots \sum_{k_j = 1}^{k_{j-1} - 1} \rho^{t - j - k_j} \norm{\tvarphi_{k_j} - \varphi_{k_j}}_2.
\end{align*}
This recursive process terminates when there does not exist an assignment of indices $k_1,\dots, k_{j-1}$ such that the summand $\sum_{k_j = 1}^{k_j - 1}$ is non-empty, i.e.\ $1 = k_j > k_{j-1} - 1$. The largest $j$ such that the aforementioned summand is possibly non-empty is when $k_i = k_{i-1} - 1$ for all $i < j$ and $k_1 = t-1$, which implies $k_i = t - i$. In order for $k_j \geq 1$, we must have $k_{j-1} - 1 = t - j \geq 1$, i.e.\ $j \leq t - 1$, and thus our recursive expansion terminates when $j = t-1$. Therefore, continuing our earlier series of inequalities, we have
\begin{align*}
    \norm{\tvarphi_t - \varphi_t}_2 &\leq \gamma\varepsilon \sum_{k_1=0}^{t-1}\rho^{t-1} \beta\norm{\xi}_2  + \cdots  + (\gamma\varepsilon)^{t-1} \sum_{k_1 = 1}^{t-1} \sum_{k_2 = 1}^{k_1 - 1}\cdots \sum_{k_{t-1} = 0}^{k_{t-2} - 1} \rho^{1 - k_{t-1}} \beta \norm{\xi}_2 \\
    &\quad + (\gamma\varepsilon)^{t-1} \sum_{k_1 = 1}^{t-1} \sum_{k_2 = 1}^{k_1 - 1}\cdots \sum_{k_{t-1} = 1}^{k_{t-2} - 1}  \rho^{t - (t-1) - k_{t-1}} \norm{\tvarphi_{k_{t-1}} - \varphi_{k_{t-1}}}_2 \\
    &= \gamma\varepsilon \sum_{k_1=0}^{t-1}\rho^{t-1} \beta\norm{\xi}_2  + \cdots  \\
    &\quad + (\gamma\varepsilon)^{t-1} \sum_{k_1 = 1}^{t-1} \sum_{k_2 = 1}^{k_1 - 1}\cdots \sum_{k_{t-1} = 0}^{k_{t-2} - 1} \rho^{1 - k_{t-1}} \beta \norm{\xi}_2 + (\gamma \varepsilon)^{t-1} \norm{\tvarphi_1 - \varphi_1}_2 \\
    &\leq \gamma\varepsilon \sum_{k_1=0}^{t-1}\rho^{t-1} \beta\norm{\xi}_2  + \cdots \\
    &\quad + (\gamma\varepsilon)^{t-1} \sum_{k_1 = 1}^{t-1} \sum_{k_2 = 1}^{k_1 - 1}\cdots \sum_{k_{t-1} = 0}^{k_{t-2} - 1} \rho^{1 - k_{t-1}} \beta \norm{\xi}_2 + (\gamma \varepsilon)^{t-1} \gamma\varepsilon\rho\beta\norm{\xi}_2 \\
    &\leq \beta \norm{\xi}_2 \sum_{j=1}^{t-1} (\gamma \varepsilon)^{j} \rho^{t-j} \paren{\sum_{k_1 = 1}^{t-1} \cdots \sum_{k_{j-1} = 1}^{k_{j-2} -1} \sum_{k_j = 0}^{k_{j-1} - 1} 1 } + \beta\norm{\xi}_2 (\gamma \varepsilon)^{t} \rho \\
    &\leq \beta \norm{\xi}_2 \sum_{j=1}^{t-1} (\gamma \varepsilon)^{j} \rho^{t-j} \paren{\sum_{k_1 = 1}^{t-1} \cdots \sum_{k_{j-1} = 1}^{k_{j-2} -1} \sum_{k_j = 0}^{k_{j-1} - 1} 1 } + \beta\norm{\xi}_2 (\gamma \varepsilon)^t + \beta\norm{\xi}_2 \rho^t.
\end{align*}
Now it remains to determine the value of $\sum_{k_1 = 1}^{t-1} \cdots \sum_{k_{j-1} = 1}^{k_{j-2} -1} \sum_{k_j = 0}^{k_{j-1} - 1} 1$. Observe the sum is only non-empty if for each $1 \leq i \leq j$, $k_{i-1} - 1 - k_i \geq 0$, where we define $k_0 = t$. Let us define the variables $c_i = k_{i-1} - k_i \geq 1$ for $i = 1,\dots, j$, and we define $c_{j+1} := k_j - 0 = k_j \geq 0$. The tuple $(c_1,\dots,c_{j+1})$ thus satisfies $\sum_{i=1}^{j} c_i = \sum_{i=1}^{j+1} k_{i-1} - k_i = t - 0 = t$. Therefore, the number of terms in the nested summand is equal to the number of integer tuples $(c_1,\dots,c_{j+1})$, where $c_1,\dots,c_j$ are positive and $c_{j+1} \geq 0$, that sum up to $t$, which in turn can be transformed into a balls-and-bins problem where we have $t$ total balls, and $j+1$ bins, but with the first $j$ bins already containing $1$ ball. Thus applying the standard balls-and-bins formula for $t-j$ balls and $j+1$ bins, we get
\begin{align*}
    \sum_{k_1 = 1}^{t-1} \cdots \sum_{k_{j-1} = 1}^{k_{j-2} -1} \sum_{k_j = 0}^{k_{j-1} - 1} 1 &= \binom{(t-j) + (j+1) - 1}{(j+1) - 1} = \binom{t}{j}.
\end{align*}
Plugging this back into the last line of our series of inequalities, we get
\begin{align*}
    \norm{\tvarphi_t - \varphi_t}_2 &\leq \beta \norm{\xi}_2 \sum_{j=1}^{t-1} (\gamma \varepsilon)^{j} \rho^{t-j} \paren{\sum_{k_1 = 1}^{t-1} \cdots \sum_{k_{j-1} = 1}^{k_{j-2} -1} \sum_{k_j = 0}^{k_{j-1} - 1} 1 } + \beta\norm{\xi}_2 (\gamma \varepsilon)^t + \beta\norm{\xi}_2 \rho^t\\
    &= \beta\norm{\xi}_2 \sum_{j=1}^{t-1}(\gamma \varepsilon)^{j} \rho^{t-j} \binom{t}{j} + \beta\norm{\xi}_2 (\gamma \varepsilon)^t + \beta\norm{\xi}_2 \rho^t \\
    &= \beta\norm{\xi}_2 \sum_{j=0}^{t} \binom{t}{j} (\gamma \varepsilon)^{j} \rho^{t-j} \\
    &= \beta\norm{\xi}_2 \paren{\rho + \gamma \varepsilon}^t.
\end{align*}
This gives us the bound:
\begin{align*}
    \abs{\tildh(\xi, V) - h(\xi, V)} &\leq L_V \beta B_X (\rho + \gamma\varepsilon + \eta^2) + \nu.
\end{align*}

\item \textbf{Combined Adversary} $\Delta_{\varepsilon_x, \varepsilon_u}^{x,u}$: from $(\beta, \rho, \gamma)$-E$\delta$ISS, we have
\begin{align*}
    \norm{\tvarphi_t - \varphi_t}_2 &\leq \gamma \sum_{k=0}^{t-1} \rho^{t-1-k} \norm{\delta^x(\tvarphi_k) + \delta^u_t}_2 \\
    &\leq\gamma \sum_{k=0}^{t-1} \rho^{t-1-k} \norm{\delta^x(\tvarphi_k)}_2 + \gamma \sum_{k=0}^{t-1} \rho^{t-1-k} \norm{\delta^u_t}_2 \\
    &\leq \gamma\sum_{k=0}^{t-1} \rho^{t-1-k} \norm{\delta^x(\tvarphi_k)}_2 + \gamma \frac{1}{1-\rho} \varepsilon_u \\
    &\leq \gamma \varepsilon_x \sum_{k=0}^{t-1}\rho^{t-1-k} \paren{\norm{\tvarphi_k - \varphi_k}_2 + \norm{\varphi_k}_2} + \gamma \frac{1}{1 - \rho} \varepsilon_u\\
    &\leq \gamma \varepsilon_x \paren{\max_t \norm{\tvarphi_t - \varphi_t}_2} \frac{1}{1-\rho} + \gamma \varepsilon_x \beta t \rho^{t-1} \norm{\xi}_2 + \gamma \frac{1}{1- \rho} \varepsilon_u.
\end{align*}
Similarly taking a maximum with respect to $t$ on both sides, we get
\begin{align*}
    \max_t \norm{\tvarphi_t - \varphi_t}_2 &\leq \gamma \varepsilon_x \paren{\max_t \norm{\tvarphi_t - \varphi_t}_2} \frac{1}{1-\rho} + \max_t \gamma \varepsilon_x \beta t \rho^{t-1} \norm{\xi}_2 + \gamma \frac{1}{1 - \rho}\varepsilon_u.
\end{align*}
It is straightforward to compute $\max_t t \rho^{t-1} = \frac{1}{e \rho\log(\rho^{-1})}$. Thus, rearranging some terms, and assuming that $\rho + \gamma \varepsilon_x < 1$, we get
\begin{align*}
    \max_t \norm{\tvarphi_t - \varphi_t}_2 &\leq \frac{1 - \rho}{1 - (\rho + \gamma \varepsilon_x)} \gamma \paren{\frac{\beta}{e \rho \log\paren{\rho^{-1}}} \norm{\xi}_2 \varepsilon_x + \frac{1}{1 - \rho}\varepsilon_u  },
\end{align*}
which yields our desired bound.

\end{itemize}

\end{proof}

\end{document}